\documentclass{article}
\PassOptionsToPackage{numbers,compress,sort&compress}{natbib}

\usepackage{microtype}
\usepackage{graphicx}
\usepackage{subfigure}
\usepackage{booktabs} 
\usepackage[subfigure]{tocloft}
\newcommand{\listappendixname}{Contents for Appendix}
\newlistof{appendix}{apc}{\listappendixname}
\usepackage{appendix}
\usepackage{threeparttable}

\usepackage{hyperref}
\usepackage{enumitem}
\setlist{topsep=1.0pt,itemsep=0.5pt,parsep=0pt,partopsep=0pt}
\usepackage{cancel}
\usepackage{multirow}

\usepackage[normalem]{ulem}

\usepackage{algorithm}
\usepackage{algorithmic}
\usepackage{amssymb} 

\newcommand{\RComment}[1]{\hfill$\triangleright$~#1}

\usepackage[preprint]{icml2026}

\usepackage{amsmath}
\usepackage{mathtools}
\usepackage{amsthm}

\makeatletter
\def\thm@space@setup{%
  \thm@preskip=2pt plus 0pt minus 0.5pt%
  \thm@postskip=1.5pt plus 0pt minus 1pt%
}
\makeatother

\makeatletter
\renewcommand\paragraph{\@startsection{paragraph}{4}{\z@}%
  {0.3ex plus 0.2ex minus 0.3ex}
  {-1em}
  {\normalfont\normalsize\bfseries}}
\makeatother

\usepackage{bbding}
\usepackage{caption}
\usepackage[capitalize,noabbrev]{cleveref}

\theoremstyle{plain}
\newtheorem{theorem}{Theorem}[section]
\newtheorem{proposition}[theorem]{Proposition}
\newtheorem{lemma}[theorem]{Lemma}

\theoremstyle{definition}

\theoremstyle{remark}

\newtheorem*{lemma*}{Lemma}
\newtheorem*{proposition*}{Proposition}
\newtheorem*{proof*}{Proof}
\newtheorem*{theorem*}{Theorem}
\newtheorem*{corollary*}{Corollary}

\newcommand{\yunfei}{\textcolor{black}{}}

\newcommand{\dd}{\mathrm{d}}

\newcommand{\mbsx}{{\mathbf{x}}}

\newcommand{\mbsy}{{\mathbf{y}}}
\newcommand{\mbxim}{{\mathbf{X}^{\text{imp}}}}

\newcommand{\mbxobs}{{\mathbf{X}^{\text{obs}}}}
\newcommand{\mbsxim}{{\mathbf{x}^{\text{imp}}}}

\newcommand{\mbsxobs}{{\mathbf{x}^{\text{obs}}}}
\usepackage[textsize=tiny]{todonotes}

\icmltitlerunning{Analyzing and Improving Diffusion Models for Time-Series Data Imputation: A Proximal Recursion Perspective}

\begin{document}

\twocolumn[
\icmltitle{Analyzing and Improving Diffusion Models for Time-Series Data Imputation:\\A Proximal Recursion Perspective
}


\icmlsetsymbol{equal}{*}
 \begin{icmlauthorlist}
    \icmlauthor{Zhichao Chen}{pku}
        \icmlauthor{Hao Wang}{zju}
    \icmlauthor{Fangyikang Wang}{zju}
        \icmlauthor{Licheng Pan}{zju}
      \icmlauthor{Zhengnan Li}{cuhk}
    \icmlauthor{Yunfei Teng}{pku,baai}

    \icmlauthor{Haoxuan Li}{pkuDS}
    \icmlauthor{Zhouchen Lin}{pku,pkuIAI,pazhouLab}
  \end{icmlauthorlist}

  \icmlaffiliation{pku}{State Key Lab of General AI, School of Intelligence Science and Technology, Peking University.}
  \icmlaffiliation{zju}{Zhejiang University.}
  \icmlaffiliation{baai}{Beijing Academy of Artificial Intelligence (BAAI).}
  \icmlaffiliation{cuhk}{The Chinese University of Hong Kong, Shenzhen.}
  \icmlaffiliation{pkuDS}{Center for Data Science, Peking University.}
   \icmlaffiliation{pkuIAI}{Institute for Artificial Intelligence, Peking University.}
   \icmlaffiliation{pazhouLab}{Pazhou Laboratory (Huangpu), Guangzhou, China.}

\icmlcorrespondingauthor{Zhouchen Lin}{zlin@pku.edu.cn}

\icmlkeywords{Machine Learning, ICML}

\vskip 0.3in
]



\printAffiliationsAndNotice{\icmlEqualContribution} 

\begin{abstract}
Diffusion models (DMs) have shown promise for Time-Series Data Imputation (TSDI); however, their performance remains inconsistent in complex scenarios. We attribute this to two primary obstacles: (1) non-stationary temporal dynamics, which can bias the inference trajectory and lead to outlier-sensitive imputations; and (2) objective inconsistency, since imputation favors accurate pointwise recovery whereas DMs are inherently trained to generate diverse samples. To better understand these issues, we analyze DM-based TSDI process through a proximal-operator perspective and uncover that an implicit Wasserstein distance regularization inherent in the process hinders the model’s ability to counteract non-stationarity and dissipative regularizer, thereby amplifying diversity at the expense of fidelity. 
Building on this insight, we propose a novel framework called SPIRIT ($\underline{\text{S}}$emi-$\underline{\text{P}}$rox$\underline{\text{i}}$mal Transport $\underline{\text{R}}$egularized time-series $\underline{\text{I}}$mpu$\underline{\text{t}}$ation). Specifically, we introduce entropy-induced Bregman divergence to relax the mass preserving constraint in the Wasserstein distance, formulate the semi-proximal transport (SPT) discrepancy, and theoretically prove the robustness of SPT against non-stationarity. Subsequently, we remove the dissipative structure and derive the complete SPIRIT workflow, with SPT serving as the proximal operator. Extensive experiments demonstrate the effectiveness of the proposed SPIRIT approach.
\end{abstract}

\section{Introduction}


Data completeness in time series is paramount across many domains~\citep{wang2025optimal}. For example, in healthcare~\citep{prosperi2020causal}, clinicians rely on wearable and ambient sensors to continuously monitor patients’ physiological signals, yet the collected records are often incomplete due to sensor disconnections. Similarly, in industrial manufacturing~\citep{10508098,wang2025inverse}, sensor networks are deployed to support process monitoring; however, harsh operating conditions and mechanical vibrations can cause sensor malfunctions and missing readings. Such incompleteness undermines data integrity, which is essential for accurate analytics~\citep{qiu2024tfb,liusundial}, thereby underscoring the need for effective time-series data imputation (TSDI) techniques.

Recently, diffusion models (DMs) have been widely adopted for TSDI task due to their excellent performance for data generation tasks~\citep{yang2023diffusion}. Given partially observed time-series data, DMs define a forward diffusion (noising) process that gradually perturbs the data with Gaussian noise, producing a sequence of increasingly noisy variables. A neural network is introduced to estimate the score function, and TSDI can be reformulated as sampling from the conditional distribution by solving the reverse-time stochastic differential equation, which imputes the missing values via progressively denoises the variables from the initial noise. Initiated by \citet{tashiro2021csdi}, subsequent DM-based methods for TSDI have mainly focused on refining the model architecture~\citep{10184808,10508098}, redesigning the forward noising process~\cite{chen2023provably}, and improving the training objectives~\citep{glocalImputation,yu2025missing}. Owing to DMs' strong ability to generate high-quality samples, diffusion models have become a prevalent approach for time-series data imputation.

Despite the success of DMs, we argue that directly applying DMs to TSDI can lead to suboptimal performance due to two latent limitations. First, \emph{non-stationarity}: time-series data often exhibit non-stationary fluctuations~\citep{wang2025optimal,liu2022non}, which are not explicitly considered in existing stochastic differential equation-based diffusion formulations to our knowledge. Second, the \emph{objective inconsistency}: DMs are primarily designed for data generation and thus favor \emph{diverse} samples, whereas imputation prioritizes \emph{accurate} recovery of the missing values~\citep{selfSupervisedDiffusionImputation,chen2024rethinking}. 

To address these issues, we first \emph{analyze} DM-based TSDI from through the lens of proximal recursion, identify two key terms, namely the Wasserstein distance and a dissipative regularizer that contribute to the above problems. Building on this analysis, we mitigate these issues by relaxing the Wasserstein distance with a generalized Bregman divergence induced by the entropy functional and removing the dissipative regularizer. We then re-derive an \emph{improved} framework, termed \underline{\text{S}}emi-$\underline{\text{P}}$rox$\underline{\text{i}}$mal Transport $\underline{\text{R}}$egularized time-series $\underline{\text{I}}$mpu$\underline{\text{t}}$ation (SPIRIT), for the TSDI task.

\noindent\textbf{Contributions:} The main contributions of this manuscript can be summarized as follows:
\begin{enumerate}[leftmargin=*]
    \item{We cast DM-based TSDI as an optimization problem proximal term and pinpoint two key bottlenecks: (i) limited robustness to non-stationarity induced by the Wasserstein distance, and (ii) imputation inaccuracy caused by the dissipative regularizer.
    }
    \item{We relax the Wasserstein distance via the generalized Bregman divergence and propose a semi-proximal transport discrepancy; we further provide theoretical analysis on robustness to non-stationarity.}
    \item{We remove the dissipative regularizer and re-derive a TSDI procedure under the proximal optimization framework with the semi-proximal transport discrepancy, yielding a new DM-based TSDI termed SPIRIT.}
\end{enumerate}

\section{Preliminaries}\label{sec:preliminariesInformation}

As a preliminary note, this study focuses on TSDI task \emph{per se}, specifically to estimate the most probable values of the missing entries. We do not view imputation as a way to generate inputs for downstream tasks~\citep{jarrett2022hyperimpute}, such as training regression models for label prediction \citep{9210118,zhao2023transformed} or using pseudo-labels for unbiased learning~\citep{li2024relaxing}. In such settings, imputation may require joint training to optimize task-specific objectives~\citep{wang2025inverse}. In addition, in this manuscript, we mainly focus on the missing completely at random (MCAR) setting to facilitate the theoretical analysis. Due to page limit, detailed preliminaries to understand this manuscript is provided in~\Cref{sec:additionalBackgroundKnowledge} in the appendix. 

\subsection{Problem Formulation}\label{subsec:problemFormulationResult}
Suppose $\mathbf{X}^{\text{ideal}} \in \mathbb{R}^{N\times T \times D}$ denotes the fully observed time series with $N$ pieces of data, and each datum is consists of $T$ chronologically ordered observations and $D$ features. Missing entries are encoded by a binary mask $\mathbf{M} \in \{0,1\}^{N\times T \times D}$, where $\mathbf{M}_{n,t,d} = 1$ if the entry $\mathbf{X}^{\text{ideal}}_{n,t,d}$ is missing and $\mathbf{M}_{n,t,d} = 0$ otherwise. The observed data matrix $\mathbf{X}^{\text{obs}}$ is then given by $\mathbf{X}^{\text{obs}}= \mathbf{X}^{\text{ideal}} \odot (1 - \mathbf{M}) + \texttt{NaN} \odot \mathbf{M}$, where $\odot$ denotes the Hadamard product and ``\texttt{NaN}'' represents unobserved entries. The goal of time-series imputation is to construct an imputed data matrix $\mathbf{X}^{\text{imp}} \in \mathbb{R}^{N\times T \times D}$ from $\mathbf{X}^{\text{obs}}$ such that $\mathbf{X}^{\text{imp}}\odot\mathbf{M}+ \mathbf{X}^{\text{ideal}}\odot(1-\mathbf{M})\approx \mathbf{X}^{\text{ideal}}$. For $\mathbf{M}_{n,t,d}=0$, we have $\mathbf{X}^{\text{imp}}\odot (1-\mathbf{M}) = \mathbf{X}^{\text{obs}}\odot (1-\mathbf{M}) $. On this basis, we denote probability density functions (PDFs) by $p(\cdot)$; for example, the PDF of the fully observed time series is written as $p(\mathbf{X}^{\text{ideal}})$. During imputation, the model induces an empirical distribution over the imputed samples, which we denote by $q(\mathbf{X}^{\text{imp}})$. Concretely, we represent $q$ as a Dirac delta measure, $q(\mathbf{X}^{\text{imp}}) = \frac{1}{N} \sum_{n=1}^{N} \delta_{\mathbf{x}^{\text{imp}}_n}$, where $\delta_{\mathbf{x}^{\text{imp}}_n}$ is the Dirac measure concentrated at the $n$-th imputed observation $\mathbf{x}^{\text{imp}}_n$. Since time-series data are indexed along the temporal axis and the SDE underlying DMs is also time-defined, we use $\tau$ and $t$ to denote the time indices, and $T$ and $\mathrm{T}$ to denote the terminal time indices of the dataset and the DM. 


\subsection{Proximal Operator and Proximal Recursion}
Given a proper, lower semi-continuous convex function $g: \mathbb{R}^D \rightarrow \mathbb{R} \cup \{+\infty\}$, the proximal operator $\text{prox}_{\varepsilon g}(\cdot)$ with step size $\varepsilon$ is given by:
\begin{equation}\label{eq:proximalOperationResult}
     \text{prox}_{\varepsilon g}(\mbsx) = \mathop{\arg\min}_{\mbsy} g(\mbsy) + \frac{1}{2\varepsilon}\|\mbsy-\mbsx\|_2^2.
\end{equation}
Crucially,~\Cref{eq:proximalOperationResult} balances minimizing $g$ against a ``proximal term'', $\frac{1}{2\varepsilon}\|\mbsy-\mbsx\|_2^2$, which enforces the solution $\mbsy$ to remain in the neighborhood of $\mbsx$ with a strength controlled by the coefficient $\frac{1}{2\varepsilon}$~\citep{parikh2014proximal}, and the process that iteratively repeating~\Cref{eq:proximalOperationResult} is named proximal recursion~\citep{8890903,caluya2021wasserstein}.

\subsection{Wasserstein Distance} 
Let $\mathcal{P}_2(\mathbb{R}^D)$ be the set of probability measures on $\mathbb{R}^D$ with finite second moment. For $\mu,\nu \in \mathcal{P}_2(\mathbb{R}^D)$, the $2$-Wasserstein distance~\citep{villani2009optimal} is defined by the following ``optimal transport'' (OT) problem:
\begin{equation}
    \mathbb{W}_2^2(\mu,\nu)\coloneqq \inf_{\pi\in \Pi(\mu,\nu) }
    \int \|\mbsx-\mbsy\|^2 \mathrm{d}\pi(\mbsx,\mbsy),
\end{equation}
where $\Pi(\mu,\nu)$ denotes the set of couplings of $\mu$ and $\nu$, i.e., joint distributions on
$\mathbb{R}^D \times \mathbb{R}^D$ with marginals $\mu$ and $\nu$. 

\subsection{Dissipative Structure in Stochastic Dynamics}
In continuous-time stochastic dynamics~\citep{sarkka2019applied}, the evolution of a state $\mbsx_\tau$ is modeled by
\begin{equation}\label{eq:itoSDEResults}
    \mathrm{d}\mbsx_\tau = b(\mbsx_\tau, \tau)\mathrm{d}\tau + \sigma(\mbsx_\tau, \tau)\mathrm{d}W_\tau,
\end{equation}
where $b(\mbsx_\tau, \tau)$ is the drift term, $\sigma(\mbsx_\tau, \tau)$ is the volatility term, and $\mathrm{d}W_\tau$ denotes the Wiener process. From the perspective of non-equilibrium stochastic dynamics, both the deterministic drift $b(\mbsx_\tau,\tau)=-\mbsx_\tau$ and the stochastic diffusion $\mathrm{d}W_\tau$ contribute to the dissipative behavior~\citep{Otto31012001} of the system: the $-\mbsx_\tau$ pulls the state back toward the origin energy dissipation, while the Wiener-driven diffusion $\mathrm{d}W_\tau$ spreads the state distribution to prevent collapse.


\section{Methodology}
\subsection{Motivation Analysis}
DMs implicitly introduce Wasserstein distance regularization and a dissipative structure during inference (see our further derivations). However, in the \yunfei{Time-Series Data Imputation} (TSDI) task, these properties encounter two major challenges arising from the characteristics of the data and the evaluation protocol. First, the data are highly \emph{non-stationary}: for example, in the Electricity dataset, consumption patterns differ substantially across weekdays and holidays, while in the Weather dataset, climatological patterns vary markedly across seasons. Second, the task demands stringent pointwise \emph{accuracy}, as exemplified by mean squared error (MSE) and mean absolute error (MAE). To better illustrate how these challenges interact with the diffusion framework, we present the following two toy case studies.
\paragraph{Wasserstein Distance Meets Non-Stationary.} We illustrate the challenge of non-stationarity in \Cref{subfig:vanillaDataDistribution}, where the data exhibits distinct multimodal structures accompanied by transient outliers. However, directly applying the canonical Wasserstein distance leads to severe misalignment, as demonstrated in \Cref{subfig:matchingWass}. Due to the rigorous mass conservation constraint, the transport plan generates erroneous couplings: the right mode of the source $\mu$ is incorrectly matched with the distant left mode of the target $\nu$, while significant mass is also forcibly diverted to the outliers in the upper-middle region.
\begin{figure}[htbp]
    \centering
      \vspace{-0.3cm}
      \subfigure[Data distribution .\label{subfig:vanillaDataDistribution}]{\includegraphics[width=0.31\linewidth]{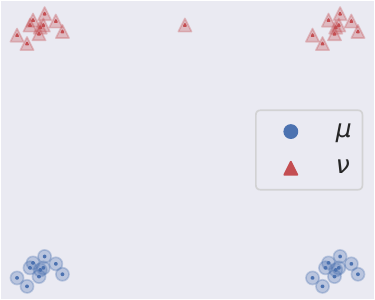}}
    \subfigure[OT plan.\label{subfig:matchingWass}]{\includegraphics[width=0.31\linewidth]{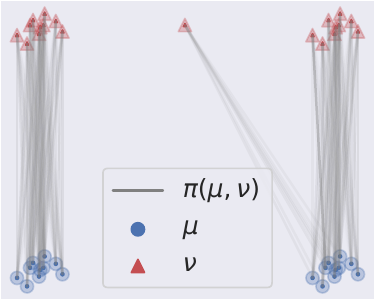}}
    \subfigure[SPT plan.\label{subfig:matchingFisherRao}]{\includegraphics[width=0.31\linewidth]{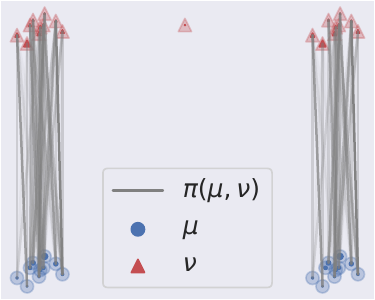}}
    \caption{Transport plan comparison between OT and SPT.}\label{fig:comparisonTransportStrategy}
\end{figure}

\paragraph{Dissipative Structures Meets Optimization.} In this part, we investigate how the dissipative structure affects imputation accuracy. Since the inference process of diffusion models is the time reversal of the predefined SDE in~\Cref{eq:itoSDEResults}, we first consider the optimal imputed value and its target distribution, as illustrated in~\Cref{subfig:vanillaDataDistributionPoint}. We then examine the effect of the dissipative terms, namely the linear drift $- \mbsxim$ and the stochastic diffusion $\mathrm{d}W_\tau$, in~\Cref{subfig:ouDisDistribution,subfig:langDisDistribution}. From these figures, we observe that the imputed value $\mbsxim$ (the star, computed as the median of the inferred samples~\citep{tashiro2021csdi} shown as white dots) can deviate significantly from the ideal value (the triangle). This indicates that, although the dissipative structure promotes diversity in the inferred samples, it can also lead to dispersed imputations and thus reduced accuracy.


\begin{figure}[htbp]
    \centering
      \vspace{-0.3cm}
      \subfigure[Expected $\mathbf{x}^{\text{ideal}}$.\label{subfig:vanillaDataDistributionPoint}]{\includegraphics[width=0.32\linewidth]{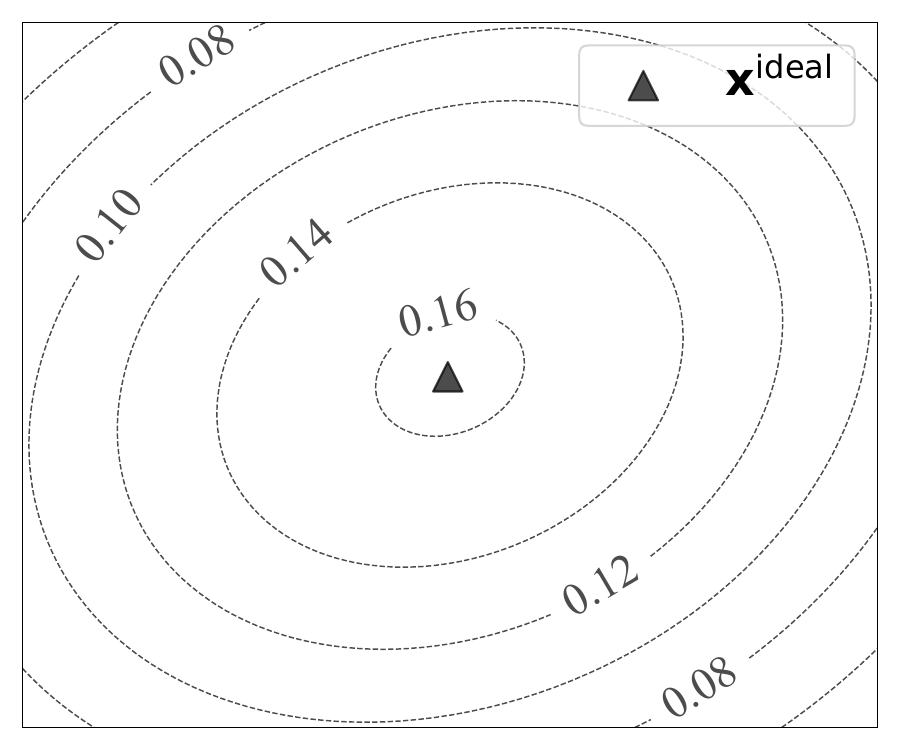}}
    \subfigure[With $\mathrm
{d}W_\tau$.\label{subfig:ouDisDistribution}]{\includegraphics[width=0.32\linewidth]{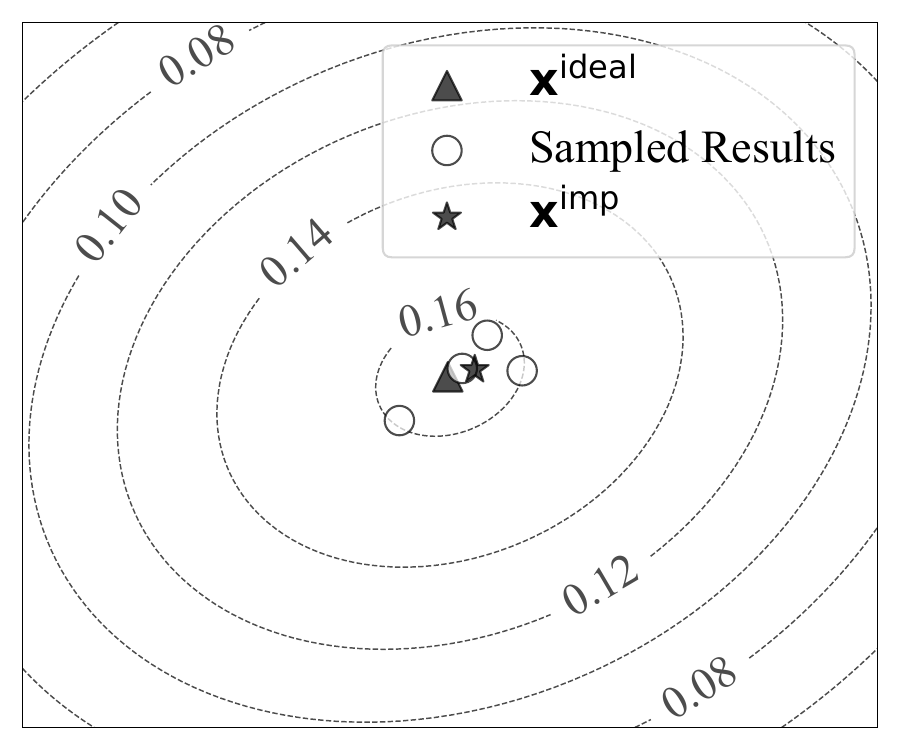}}
    \subfigure[With $-\mathbf{x}^{\text{imp}}$.\label{subfig:langDisDistribution}]{\includegraphics[width=0.32\linewidth]{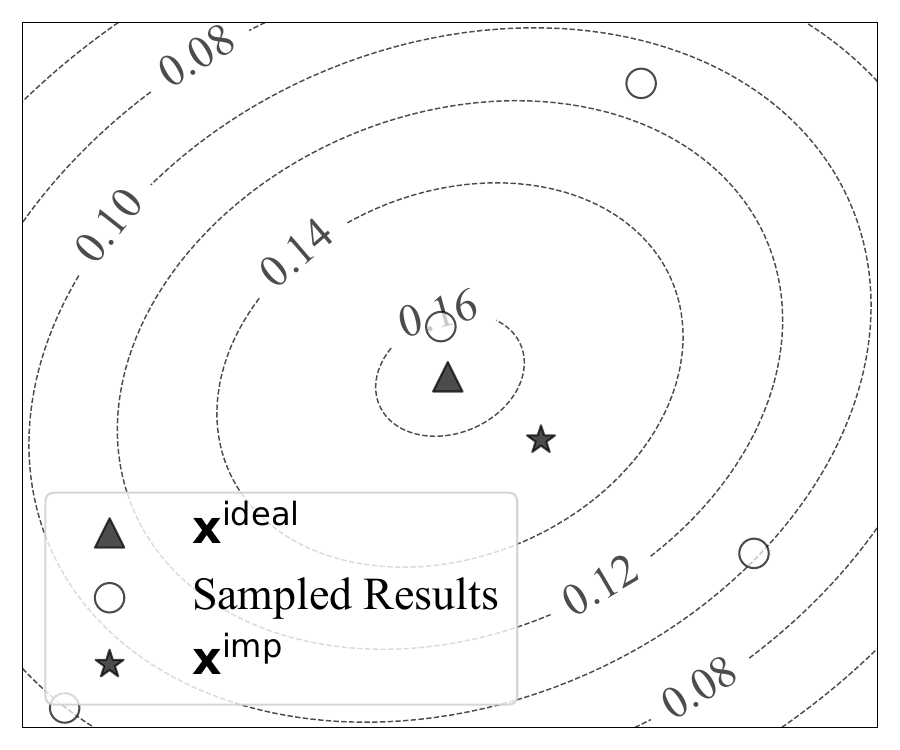}}
    \caption{Imputation results comparison vary dissipative structures.}
\end{figure}



\subsection{Analysis DMs In the Context of TSDI}\label{subsec:analysisBottleNeckResults}
In the context of DM-based TSDI, suppose that we have learned the conditional score function $\nabla\log p(\mbsx^{\mathrm{imp}}| \mbsx^{\mathrm{obs}})$. The imputation procedure is then carried out by simulating the following SDE from $\tau=\mathrm{T}$ to $\tau=0$~\citep{tashiro2021csdi,chen2023provably}:
\begin{equation}\label{eq:imputationIterativelySDEResult}
    \dd \mbsxim = [f(\mbsxim,\tau) - g^2 \nabla\log{p(\mbsx^{\text{imp}}\vert \mbsx^{\text{obs}})}]\dd \tau+g\dd W_\tau,
\end{equation}
where $f(\mbsxim,\tau)$ and $g$ are the predefined drift and diffusion coefficients, respectively. In particular, we merely update the place where $\mathbf{M}_{n,t,d}=1$. 

However, it should be pointed out that directly simulating \Cref{eq:imputationIterativelySDEResult} can be suboptimal for TSDI. To support this claim, we first state the following proposition:
\begin{proposition}\label{thm:mmsProblemSolving}
The imputation process for DMs can be formulated as iteratively solving the following optimization problem in a proximal operator form:
\begin{equation}
\begin{aligned}
\mathop{\inf}_{q'} \indent -\mathbb{E}_{q'}&[\log{p(\mbsx^{\text{imp}}\vert \mbsx^{\text{obs}})}]+ \frac{1}{2\eta} \mathbb{W}_2^2(q', q) +\phi(\mbsx^{\text{imp}}) ,
\end{aligned}
\end{equation}
where we abbreviate the candidate distribution  (distribution for current iteration) $q'(\mbsx^{\text{imp}})$ and the base distribution (distribution for previous iteration) $q(\mbsx^{\text{imp}})$ as $q'$ and $q$, respectively. The term $\frac{1}{\eta}$ is a predefined positive term determined by the noise schedule of DMs, and the term $\phi(\mbsx^{\text{imp}}) $ denotes the dissipative structure-related term, which depends on the underlying SDE and is specified as follows:
\begin{itemize}[leftmargin=*]
\item{Variance Preserving SDE (VP-SDE):
$ \phi(\mbsx^{\text{imp}}) =  \frac{1}{2}\mathbb{E}_{q'}[\log{q'(\mbsx^{\text{imp}})}]-\frac{1}{4} \mathbb{E}_{q'}[\Vert \mbsx^{\text{imp}} \Vert_2^2]$, and $\eta=\beta(\tau)$.
}
 \item{Variance Exploding SDE (VE-SDE): $  \phi(\mbsx^{\text{imp}}) =  \frac{1}{2}\mathbb{E}_{q'}[\log{q'(\mbsx^{\text{imp}})}]$, and $\eta=\frac{1}{2}\frac{\dd \sigma^2(\tau)}{\dd \tau}$.
 }
\end{itemize}
\end{proposition}
From the proposition above, it is evident that directly deploying DMs for TSDI tasks faces two critical challenges. First, the proximal term induced by the squared 2-Wasserstein distance $\mathbb{W}_2^2(q', q)$ restricts the model's flexibility, resulting in a lack of robustness toward non-stationary dynamics. Second, the proximal term induced by the dissipative regularizer $\phi(\mbsx^{\text{imp}})$ injects entropy that encourages over-diverse outputs, potentially degrading the deterministic accuracy required for TSDI task. 

The key to addressing the first issue is to devise a novel discrepancy metric that mitigates the impact of non-stationarity. Building on this, we aim to remove the dissipative structure and re-derive an alternative inference procedure that ensures accurate imputation. The subsequent two subsections therefore focus on these two aspects.



\subsection{Semi-Proximal Transport Framework}\label{subsec:sptFramework}
To address the limitations of vanilla OT framework, which is sensitivity to outliers due to strict mass conservation in non-stationary environments, we propose the SPT framework. Unlike standard OT, which forces a complete match between distributions, SPT relaxes the hard constraint on the target marginal by introducing a divergence penalty. 

Formally, we define the SPT discrepancy as follows:
\begin{equation} \label{eq:spt_def}
\begin{aligned}
    \mathbb{S}(\mu,\nu)\coloneqq \inf_{\pi \in \Pi(\mu)}
    \int \|\mathbf{x}-\mathbf{y}\|_2^2 \mathrm{d}\pi(\mathbf{x},\mathbf{y}) +   
    D_{\psi}(\pi_{\mathbf{y}}, \nu),
\end{aligned}
\end{equation}
where $\Pi(\mu)$ denotes the set of couplings with the first marginal fixed to $\mu$, and $\pi_{\mathbf{y}}$ denotes the second marginal of $\pi$. The term $D_{\psi}(\cdot, \cdot)$ is the generalized Bregman divergence defined as follows~\citep{blondel2022learning}:
\begin{equation}
    D_{\psi}(\rho, \nu) \coloneqq \psi(\rho) - \psi(\nu) - \langle \delta_\nu \psi(\nu), \rho - \nu \rangle,
\end{equation}
where $\delta_\nu \psi(\nu)$ is the first variation of $\psi(\nu)$ with-respect-to $\nu$, and $\psi:\mathcal{P}_2(\mathbb{R}^D)\to\mathbb{R}\cup\{\infty\}$ is a strictly convex functional termed Bregman potential~\cite{blair1985problem,bregmanPotential}. While the framework admits general potentials, in this work, we specifically adopt the entropic potential $\psi(\mu) \coloneqq  \int \mu(\mbsx)[\log \mu(\mbsx) - 1] \mathrm{d}\mbsx$. We select this specific potential for its geometric properties. Specifically, its gradient map acts as a mirror map that ensures that the iterations naturally remain within the space of positive measures (i.e., $\rho \ge 0$) throughout the optimization~\citep{hsieh2018mirrored,shi2021sampling,sharrock2023learning}, which guaranties the well-definedness of the transport plan and thus stabilizes the subsequent iterative updates.

By replacing the rigorous constraint $\pi_{\mathbf{y}} = \nu$ with the soft penalty $D_{\psi}(\pi_{\mathbf{y}} , \nu)$, the SPT framework transforms the vanilla OT framework into a selective matching process. When the cost of transporting mass to an outlier is excessively high, the optimization prefers to incur the penalty rather than distorting the transport plan, thereby ensuring robustness against non-stationary scenario. Based on Lemma 1 in~\citep{fatras2021unbalanced} and Theorem C.2 in~\citep{wang2025optimal}, we state the following lemma to demonstrate the outlier robustness of the proposed SPT discrepancy:
\begin{lemma}\label{prop:robustness}
Let $\mu$ and $\nu$ be probability measures on $\mathbb{R}^D$, and consider a contaminated target distribution $\tilde{\nu} = (1-\zeta)\nu + \zeta\delta_{\mathbf{z}}, \zeta\in(0,1)$, where $\delta_\mathbf{z}$ denotes a Dirac mass at the outlier location $\mathbf{z} \in\mathbb{R}^D$. The Wasserstein distance has the following lower bound:
\begin{equation}
\begin{aligned}
     \mathbb{W}_2^2&(\mu,\tilde{\nu}) \ge \zeta\mathbb{W}_2^2(\mu,{\nu})   \\
&+(1-\zeta) [\|\mathbf{y}^*-\mathbf{z}\|_2^2 -g(\mathbf{y}^*)+ \int{h(\mathbf{x})\mu(\mathbf{x})\dd\mathbf{x}}],
\end{aligned}
\end{equation}
for some $\mathbf{y}^*$ belonging to the support of $\nu$, and where $f$ and $g$ are optimal dual potentials for $\mathbb{W}_2^2(\mu,\nu)$. Meanwhile, the SPT discrepancy with Bregman potential $\psi(\nu) \coloneqq  \int \nu(\mathbf{y})[\log \nu(\mathbf{y}) - 1] \mathrm{d}\mathbf{y}$ under target contamination admits the bound as follows:
\begin{equation}
\mathbb{S}(\mu,\tilde{\nu})\le(1-\zeta)\mathbb{S}(\mu,\nu)+\zeta(1-e^{-D(\mathbf{z}) }) + C(\zeta),
\end{equation}
where $D(\mathbf{z})\coloneqq \int\| \mathbf{z} - \mathbf{x} \|_2^2 \mu(\mathbf{x})\dd\mathbf{x} $ is the average distance of $\mathbf{z}$ and samples from $\nu$, and $C(\zeta)$ is a constant defined as $C(\zeta)\coloneqq (1-\zeta)\log\frac{1}{1-\zeta}-\zeta\log\zeta$.
\end{lemma}



This theoretical advantage is visually corroborated in the case study shown in~\Cref{fig:comparisonTransportStrategy}. As demonstrated in~\Cref{subfig:matchingFisherRao}, while standard OT incorrectly pairs distinct modes due to forced matching, the SPT framework successfully ignores transient noise and correctly aligns the stable modes, effectively handling the non-stationary property.

\subsection{Diversity-Eliminated Functional for Imputation}\label{subsec:diversityEliminatedImputationProcess}
Based on the theoretical results in~\Cref{subsec:analysisBottleNeckResults,subsec:sptFramework}, we design the following objective functional to ``improve'' the DM-based TSDI task: 
\begin{equation}\label{eq:lossFuncDensityEliminated}
\begin{aligned}
\mathop{\inf}_{q'} \indent -\mathbb{E}_{q'}[\log  p&(\mbsx^{\text{imp}}\vert  \mbsx^{\text{obs}})] + \frac{1}{2\eta} \mathbb{S}(q', q)\\
 & + \frac{1}{2} \mathbb{E}_{q'}[\| \nabla\log{p(\mbsx^{\text{imp}}\vert \mbsx^{\text{obs}})}  \|_2^2 ] ,
\end{aligned}
\end{equation}
where we eliminate the dissipative structure-related term $\phi(\mbsxim)$, which encourages the diversity of the imputation results. Based on this, we add the gradient norm regularization of the corresponding $\nabla\log{p(\mbsx^{\text{imp}}\vert \mbsx^{\text{obs}})}$, $\mathbb{E}_{q'}[\| \nabla\log{p(\mbsx^{\text{imp}}\vert \mbsx^{\text{obs}})}  \|_2^2 ]$, which seeks to finding the point that has saddle point such that we can accelerate the imputation process. As such, we formulate the following theorem to demonstrate the sufficient condition that reduces the functional given by~\Cref{eq:lossFuncDensityEliminated}:
\begin{proposition}\label{prop:updatePropositionResults}
Assume $p(\mbsx^{\text{imp}}| \mbsx^{\text{obs}})$ is once continuously differentiable and $\nabla \log p(\mbsx^{\text{imp}}| \mbsx^{\text{obs}})$ is square-integrable under the measures considered. Let $\psi(\rho)=\int \rho(\mbsx)[\log\rho(\mbsx)-1]\mathrm d\mbsx$. Represent $q'$ by an empirical measure $q'=\sum_{i=1}^N w_i\delta_{\mbsx_i}$ with $w_i\ge 0$ and $\sum_{i=1}^N w_i=1$. The descent directions $\boldsymbol{T}$ for updating the locations $\{\mbsx_i\}_{i=1}^N$ and logarithmic weights $\{\log{w_i}\}_{i=1}^N$, which yield an approximate solution to \Cref{eq:lossFuncDensityEliminated}, are given as follows:
\begin{itemize}[leftmargin=*]
    \item \textbf{Location direction}, where $\{\mbsx_i\}_{i=1}^N$ are updated by:
    \begin{equation}\label{eq:impTransportDirection}
        \boldsymbol{T}_{\mbsxim}(\mbsxim)
        = \nabla\log p(\mbsx^{\mathrm{imp}}| \mbsx^{\mathrm{obs}}).
    \end{equation}
    \item \textbf{Weight direction}, where logarithmic weights $\{\log{w_i}\}_{i=1}^N$ are updated by:
\begin{equation}\label{eq:impTeleportDirection}
\begin{aligned}
    \boldsymbol{T}_{w}(\mbsx^{\text{imp}}) =- 2 &\|\nabla \log p(\mbsx^{\text{imp}}\vert \mbsx^{\text{obs}})\|_2^2 \\
    & + 2\mathbb{E}[\| \nabla \log p(\mbsx^{\text{imp}}\vert \mbsx^{\text{obs}})\|_2^2].
\end{aligned}
\end{equation}    
\end{itemize}
\end{proposition}
Even though \Cref{prop:updatePropositionResults} provides an update direction for the log-weights $\{\log w_i\}_{i=1}^N$, directly applying this direction does not automatically guarantee that the resulting weights remain feasible, i.e., $w_i \ge 0$ and $\sum_{i=1}^N w_i = 1$. To alleviate this issue, we introduce the following proposition to ensure a well-defined iteration process for $\{w_i\}_{i=1}^N$ using the proximal recursion framework:


\begin{proposition}\label{prop:softMaxWeightScheme}
Let $\eta>0$ and $\boldsymbol{T}_w\in\mathbb{R}^N \to \mathbb{R}$. Define the intermediate log-weights: $  \log \widehat{{w}}_i^{(k+1)} \coloneqq  \log {{w}}_i^{(k)} + \eta\boldsymbol{T}(\mbsxim)$, the corresponding normalized weights can be obtained by the following equation:
\begin{equation}\label{eq:softMaxPullBackResult}
  {{w}}_i^{(k+1)} = \frac{{\widehat{{w}}_i^{(k+1)}}}{\sum_{j=1}^{D}{\widehat{{w}}_j^{(k+1)}}}.
\end{equation}
\end{proposition}
So far, within the proximal-recursion framework, we have addressed robustness to non-stationarity and the inaccuracy induced by the two proximal terms namely Wasserstein distance and dissipative regularizer in DM-based TSDI task.




\subsection{Overall Workflow}\label{subsec:overallWorkFlowSPIRIT}
Although~\Cref{subsec:diversityEliminatedImputationProcess} proposes a novel imputation process under the proximal recursion framework, the overall SPIRIT workflow is not yet explicitly summarized. 
To complete the workflow, we still need a principled way to learn the conditional score
$\nabla \log p(\mbsxim\mid \mbsxobs)$.

However, compared with previous DM-based TSDI approaches, where the score function is obtained by bridging missing and observed data using a predefined SDE and can be learned via score matching, our method modifies the objective functional and thus cannot directly adopt the same learning strategy. Therefore, we introduce the following proposition to enable the learning of $ \nabla \log p(\mbsxim \mid \mbsxobs) $:
\begin{proposition}\label{prop:scoreNetworkLearningObjective}
For learning the score network $s_\theta(\mbsxim)$, the following two objectives are equivalent:
\begin{equation}\label{eq:scoreLearningObjective}
\begin{aligned}
   &   \mathop{\arg\min}_{ s_\theta} \| s_\theta(\mbsxim) - \nabla\log{p(\mbsxim\vert\mbsxobs)}\|_2^2 \\
   = &  \mathop{\arg\min}_{ s_\theta} \| s_\theta(\mbsxim) - \nabla\log{p(\mbsxim)}\|_2^2.
\end{aligned}
\end{equation}
\end{proposition}
Following~\Cref{prop:scoreNetworkLearningObjective}, it suffices to learn the marginal score $\nabla \log p(\mbsxim)$.
Moreover, given an initial imputation $\mbsxim$, the right-hand side of
\Cref{eq:scoreLearningObjective} can be learned via the denoising score matching~\citep{vincent2011connection} given in~\Cref{eq:dsmTargetFunction}. The justification of this equivalence is provided in~\Cref{subsec:derivationDSMProp}.
\begin{equation}\label{eq:dsmTargetFunction}
\mathop{\arg\min}_{s_\theta}\mathbb{E}_{q_{\sigma}(\widehat{\mbsx}^{\text{imp}}\vert\mbsxim)}[\Vert s_\theta(\widehat{\mbsx}^{\text{imp}}) - \nabla\log q_{\sigma}(\widehat{\mbsx}^{\text{imp}} \vert\mbsxim)\Vert_2^2 ],
\end{equation}
where $\sigma$ is variance scale, $\widehat{\mbsx}^{\text{imp}}$ is obtained by $\widehat{\mbsx}^{\text{imp}}=\mbsxim + \epsilon, \epsilon\sim\mathcal{N}(\mathbf{0}, \sigma^2\mathbf{I})$, and $ \nabla \log{q_{\sigma}(\widehat{\mbsx}^{\text{imp}} \vert\mbsxim)} =-\frac{\widehat{\mbsx}^{\text{imp}} -\mbsxim }{\sigma^2}$.

Finally, the overall workflow of SPIRIT is summarized in~\Cref{algo:spiritOverall}. It consists of two stages: ``Score Learning'' and ``Recursive Imputation''. In the ``Score Learning'' stage, we train a score network $s_\theta(\mbsxim)$ to approximate $\nabla \log p(\mbsxim \mid \mbsxobs)$ by minimizing the objective in~\Cref{eq:scoreLearningObjective}. In the ``Recursive Imputation'' stage, we update the imputed values according to~\Cref{prop:updatePropositionResults,prop:softMaxWeightScheme}. Alternating between these two stages yields the final imputation $\mbxim$. In addition, \Cref{algo:spiritOverall} introduces an operator, \texttt{ApplyGrad}, which takes (i) a gradient direction, (ii) a learning rate, and (iii) the current variable to be updated, and applies a gradient-based update. This operator can be implemented using standard optimizers in deep learning backends~\citep{TorchNips}.


\begin{algorithm}[htbp]
\caption{The workflow of SPIRIT.}
\label{algo:spiritOverall}
\footnotesize
\textbf{Input}: $\mbxobs$: Observed Data. \\
\textbf{Parameter}: $\eta$: Proximal Recursion Step Size, $lr$: the learning rate for score network, $\sigma$: the coefficient for DSM, $\theta$: the parameter for score network, $\mathcal{E}_{\text{score}}$: iterative time for score network training, and $\mathcal{E}_{\text{imp}}$: iterative time for imputation.
\\
\textbf{Output}: $\mbxim$: the imputed data. 
\begin{algorithmic}[1] 
\FOR{$e = 1$ to $\mathcal{T}$}
\STATE $s_\theta \leftarrow \text{Equation}~\eqref{eq:dsmTargetFunction};$ \; \RComment{Score Learning}

\FOR{$e = 1$ to $\mathcal{E}$}
\STATE $   \boldsymbol{T}_{\mbsxim}(\mbsxim)\leftarrow\text{Equation}~\eqref{eq:impTransportDirection};$\;\RComment{Recursive Imputation}
\STATE $    \boldsymbol{T}_{w}(\mbsx^{\text{imp}}) \leftarrow\text{Equation}~\eqref{eq:impTeleportDirection};$\;
\STATE $ \log{\widehat{w}}\leftarrow \texttt{ApplyGrad}(    \boldsymbol{T}_{w}(\mbsxim) ,\eta,\log{w});$\;
\STATE $ {w}\leftarrow \text{Equation}~\eqref{eq:softMaxPullBackResult};$\;
\STATE $\mbsxim\leftarrow \texttt{ApplyGrad}(   w \boldsymbol{T}_{\mbsxim}(\mbsxim),\eta,\mbsxim);$\;
\STATE $\mbxim\leftarrow \mbxim \odot (1-\mathbf{M}) + \mbxim \odot \mathbf{M};$\;
\ENDFOR
\ENDFOR
\end{algorithmic}
\end{algorithm}

Notably, SPIRIT follows an alternating-update workflow, which enables an Expectation–Maximization-style convergence analysis~\citep{dempster1977maximum}. Due to space limitations, detailed discussions are deferred to~\Cref{subsec:discussionsOnTheConvergence}; here we summarize the main points. For the ``Score Learning'' stage, convergence follows from standard results under mild regularity conditions~\citep{bottou2018optimization}. For the ``Recursive Imputation'' stage, assume the energy $\mathcal{J}(q')=\mathbb{E}_{q'}[\log p(\mbsx^{\text{imp}}\mid \mbsx^{\text{obs}})]
+\mathbb{E}_{q'}[\bigl\|\nabla \log p(\mbsx^{\text{imp}}\mid \mbsx^{\text{obs}})\bigr\|_2^2]$
is lower bounded and smooth. Then, in the continuous-time limit $\eta\to 0$, $\mathcal{J}(q')$ decreases monotonically along the imputation iterates, and converges to a stationary point. Detailed discussions are provided in~\Cref{subsec:discussionsOnTheConvergence}.

\begin{table*}[!h]
\caption{Imputation performance comparison in terms of MAE and MSE.}\label{tab:baselineComparisonAll}
\centering \begin{threeparttable} \small\setlength{\tabcolsep}{3.2pt}\renewcommand\arraystretch{1.3} \begin{tabular}{l|cl|cl|cl|cl|cl|cl|cl} \toprule Dataset & \multicolumn{2}{c|}{ETT-h1}  & \multicolumn{2}{c|}{ETT-h2}  & \multicolumn{2}{c|}{ETT-m1}  & \multicolumn{2}{c|}{ETT-m2}  & \multicolumn{2}{c|}{Exchange} & \multicolumn{2}{c|}{Illness}  & \multicolumn{2}{c}{Traffic}  \\ \midrule Metric  & MAE& \multicolumn{1}{c|}{MSE} & MAE& \multicolumn{1}{c|}{MSE} & MAE& \multicolumn{1}{c|}{MSE} & MAE& \multicolumn{1}{c|}{MSE} & MAE& \multicolumn{1}{c|}{MSE} & MAE& \multicolumn{1}{c|}{MSE} & MAE& \multicolumn{1}{c}{MSE} \\ \midrule Crossformer & 0.265 & 0.191 & 0.483 & 0.591 & 0.703$^*$ & 0.966$^*$ & 0.740$^*$ & 1.086$^*$ & 0.318 & 0.291 & 0.281$^*$ & 0.291 & 0.303 & 0.320  \\  TimesNet & 0.246$^*$ & 0.116 & 0.214$^*$ & 0.087$^*$ & 0.147$^*$ & 0.045$^*$ & 0.163$^*$ & 0.070$^*$ & 0.240$^*$ & 0.127$^*$ & 0.289$^*$ & 0.246$^*$ & 0.301 & \uwave{0.288}  \\  PatchTST & \textbf{0.204} & \textbf{0.093} & 0.171 & 0.062 & 0.126$^*$ & \uwave{0.038} & 0.107$^*$ & 0.025$^*$ & 0.170$^*$ & 0.058$^*$ & 0.300$^*$ & 0.298 & 0.403 & 0.520  \\  Autoformer & 0.628$^*$ & 0.915$^*$ & 0.684$^*$ & 1.239$^*$ & 0.650$^*$ & 1.121$^*$ & 0.671$^*$ & 1.351$^*$ & 0.712$^*$ & 1.193$^*$ & 0.682$^*$ & 1.022$^*$ & 0.471 & 0.574  \\  ETSformer & 0.279$^*$ & 0.179$^*$ & 0.202$^*$ & 0.087$^*$ & 0.201$^*$ & 0.097$^*$ & 0.134$^*$ & 0.038$^*$ & 0.178$^*$ & 0.057$^*$ & 0.321$^*$ & 0.302$^*$ & \textbf{0.291} & \textbf{0.218}  \\  FiLM & 0.434$^*$ & 0.436$^*$ & 0.271$^*$ & 0.143$^*$ & 0.248$^*$ & 0.127$^*$ & 0.215$^*$ & 0.095$^*$ & 0.204$^*$ & 0.077$^*$ & 0.491$^*$ & 0.608$^*$ & 0.851 & 1.341  \\  DLinear & 0.251$^*$ & 0.135$^*$ & 0.216$^*$ & 0.097$^*$ & 0.211$^*$ & 0.096$^*$ & 0.191$^*$ & 0.076$^*$ & 0.199$^*$ & 0.073$^*$ & 0.254$^*$ & 0.174$^*$ & 0.364 & 0.367  \\ \midrule GP-VAE & 0.316$^*$ & 0.192$^*$ & 0.271$^*$ & 0.150 & 0.246$^*$ & 0.121$^*$ & 0.258$^*$ & 0.158$^*$ & 0.337$^*$ & 0.211$^*$ & 0.557$^*$ & 0.752$^*$ & 0.726 & 0.959  \\  CSDI & 0.248 & 0.319 & 0.259 & 1.437 & \textbf{0.117} & 0.041 & 0.085 & 0.076 & 0.087$^*$ & 0.030$^*$ & 8.568$^*$ & 400.6$^*$ & 23.58 & 1488  \\  Glocal & 0.240$^*$ & 0.121$^*$ & 0.194$^*$ & 0.076$^*$ & 0.140$^*$ & 0.044$^*$ & 0.113$^*$ & 0.026$^*$ & 0.163$^*$ & 0.051$^*$ & 0.348$^*$ & 0.314$^*$ & 0.381 & 0.452  \\ \midrule Sinkhorn & 0.752$^*$ & 0.994$^*$ & 0.752$^*$ & 1.003$^*$ & 0.754$^*$ & 1.000$^*$ & 0.749$^*$ & 0.998$^*$ & 0.826$^*$ & 0.996$^*$ & 0.715$^*$ & 1.012$^*$ & 0.780 & 1.041  \\  TDM & 0.750$^*$ & 0.992$^*$ & 0.750$^*$ & 1.000$^*$ & 0.754$^*$ & 1.000$^*$ & 0.748$^*$ & 0.998$^*$ & 0.822$^*$ & 0.991$^*$ & 0.705$^*$ & 0.997$^*$ & 0.773 & 1.033  \\  PSW-I & 0.219$^*$ & 0.112$^*$ & \uwave{0.137}$^*$ & \uwave{0.042} & 0.122$^*$ & 0.041$^*$ & \uwave{0.083} & \uwave{0.019}$^*$ & \uwave{0.024} & \uwave{0.002}$^*$ & \uwave{0.111} & \uwave{0.054}$^*$ & \uwave{0.292} & 0.329  \\ \midrule \textbf{SPIRIT (Ours)} & \uwave{0.209} & \uwave{0.108} & \textbf{0.131} & \textbf{0.038} & \uwave{0.118} & \textbf{0.037} & \textbf{0.082} & \textbf{0.016} & \textbf{0.023} & \textbf{0.002} & \textbf{0.107} & \textbf{0.045} & 0.292 & 0.335  \\   \midrule  Win Counts & 12& 12& 13& 13& 12& 13& 13& 13& 13& 13& 13& 13& 11& 9 \\   \bottomrule \end{tabular}  \begin{tablenotes}  \footnotesize \item \textit{Kindly Note}: Each entry represents the average results at $p_{\text{miss}}\in\{0.1, 0.2,0.3,0.4,0.5,0.6\}$. Best results are in \textbf{bold}; second best are in \uwave{wavy underline}. ``*'' marks the results that SPIRIT significantly outperform with $p$-value$<0.05$ over paired samples $t$-test. \end{tablenotes} \end{threeparttable}

\end{table*}
\section{Main Experimental Results}

\subsection{Setup}
\paragraph{Datasets.} We evaluate our methods using several standard public benchmarks for time-series imputation following~\citet{Timesnet}. Specifically, we use the ETT dataset (four subsets), Exchange, Illness, and Traffic. Comprehensive dataset statistics are presented in~\Cref{subsec:datasetDescription}.

\paragraph{Baselines Models.} Since this paper focuses on the TSDI task, we adopt a set of widely used TSDI methods as baselines to evaluate the effectiveness of the proposed SPIRIT framework, including Crossformer~\citep{zhang2023crossformer}, TimesNet~\citep{Timesnet}, PatchTST~\citep{niePatchTST}, Autoformer~\citep{xu2021autoformer}, ETSformer~\citep{woo2022etsformer}, FiLM~\citep{zhou2022film}, DLinear~\citep{zeng2023transformers}, GP-VAE~\citep{fortuin2020gp}, CSDI~\citep{tashiro2021csdi}, Glocal~\citep{glocalImputation}, Sinkhorn~\citep{muzellec2020missing}, TDM~\citep{zhao2023transformed}, and PSW-I~\citep{wang2025optimal}. We categorize Crossformer, TimesNet, PatchTST, Autoformer, ETSformer, FiLM, and DLinear as discriminative TSDI approaches, as they learn an internal time-series prediction model and use it to perform imputation. We categorize GP-VAE, CSDI, and Glocal as ``generative TSDI approaches''; notably, CSDI and Glocal are ``diffusion-based''. Finally, we categorize Sinkhorn, TDM, and PSW-I as ``alignment-based approaches'', since they perform imputation via distribution alignment.

\paragraph{Implementation.} Based on our preliminary notes outlined in~\Cref{sec:preliminariesInformation}, we conduct the simulate the MCAR scenario, detailed information for MCAR scenario simulation is given in. The missing ratio $p_{\text{Miss}}$ is simulated within $\{0.1, 0.2, 0.3, 0.4, 0.5,0.6\}$. The patch length in our experiment is set as $24$. Other detailed information regarding the model hyperparameters and missing data simulation protocols are provided in~\Cref{subsec:hyperParamsSettings,subsec:evaluationMetricResults,subsec:simulationMissingDataScenario}. All experiments are conducted on a workstation equipped with AMD EPYC 7742 CPUs and four NVIDIA RTX A100 GPUs.

\begin{table*}[!h]
\caption{Ablation study results in terms of MAE and MSE.}\label{tab:abliationStudyMainResult}
\centering \begin{threeparttable} \small\setlength{\tabcolsep}{1.2pt}\renewcommand\arraystretch{1.0} \begin{tabular}{c|c|llll|llll|llll|llll} \toprule \multirow{3}{*}{SPT} & \multirow{3}{*}{w/o $\mathbb{E}_{q'}[\log{q'}]$} & \multicolumn{4}{c|}{ETT-h1}                        & \multicolumn{4}{c|}{ETT-h2}                        & \multicolumn{4}{c|}{ETT-m1}     & \multicolumn{4}{c}{ETT-m2}       \\ \cmidrule{3-18} &                          & \multicolumn{2}{c}{MAE} & \multicolumn{2}{c|}{MSE} & \multicolumn{2}{c}{MAE} & \multicolumn{2}{c|}{MSE} & \multicolumn{2}{c}{MAE} & \multicolumn{2}{c|}{MSE} & \multicolumn{2}{c}{MAE} & \multicolumn{2}{c}{MSE} \\ \cmidrule{3-18}  &                          & Value     & $\Delta(\uparrow)$    & Value     & $\Delta(\uparrow)$     & Value     & $\Delta(\uparrow)$    & Value     & $\Delta(\uparrow)$     & Value     & $\Delta(\uparrow)$    & Value     & $\Delta(\uparrow)$     & Value     & $\Delta(\uparrow)$    & Value     & $\Delta(\uparrow)$    \\ \midrule\Checkmark & \XSolidBrush  & 0.237$^*$ & 13.1\% & 0.137$^*$ & 27.0\% & 0.153$^*$ & 16.7\% & 0.051$^*$ & 33.5\% & 0.137$^*$ & 16.2\% & 0.047$^*$ & 25.3\% & 0.105$^*$ & 28.1\% & 0.023$^*$ & 41.4\% \\ \XSolidBrush & \Checkmark  & 0.227$^*$ & 8.42\% & 0.133$^*$ & 23.3\% & 0.142$^*$ & 8.36\% & 0.046$^*$ & 20.1\% & 0.162$^*$ & 37.0\% & 0.053$^*$ & 42.0\% & 0.128$^*$ & 55.8\% & 0.033$^*$ & 103\% \\ \XSolidBrush & \XSolidBrush  & 0.573$^*$ & 173\% & 0.534$^*$ & 394\% & 0.532$^*$ & 305\% & 0.449$^*$ & 1066\% & 0.529$^*$ & 347\% & 0.443$^*$ & 1091\% & 0.516$^*$ & 528\% & 0.420$^*$ & 2457\% \\ \Checkmark & \Checkmark  & 0.209 & - & 0.108 & - & 0.131 & - & 0.038 & - & 0.118 & - & 0.037 & - & 0.082 & - & 0.016 & - \\ \bottomrule \end{tabular} \begin{tablenotes}  \footnotesize \item \textit{Kindly Note}: Each entry represents the average results at six missing ratios: $p_{\text{miss}}\in\{0.1, 0.2,0.3,0.4,0.5,0.6\}$. $\Delta(\uparrow)$ denotes performance degeneration percentage compared to SPIRIT framework. ``*'' marks the results that SPIRIT significantly outperform with $p$-value$<0.05$ over paired samples $t$-test. \end{tablenotes} \end{threeparttable}

\end{table*}

\begin{figure*}[htbp]
    \centering
      \subfigure[Sensitivity analysis on $\eta$ .\label{subfig:sens_eta}]{\includegraphics[width=0.245\linewidth]{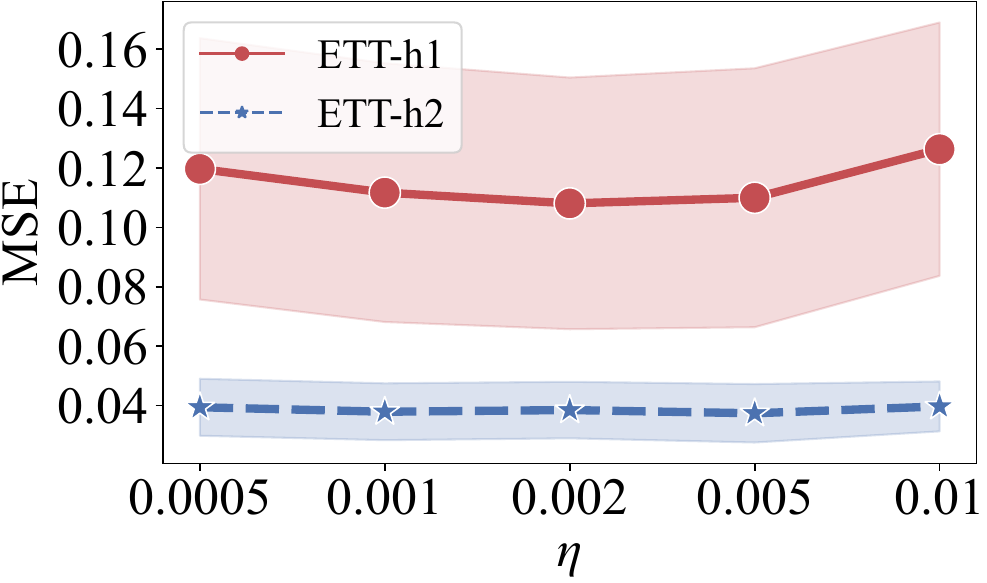}}
     \subfigure[Sensitivity analysis on $\mathcal{E}$ .\label{subfig:sens_iter_times}]{\includegraphics[width=0.245\linewidth]{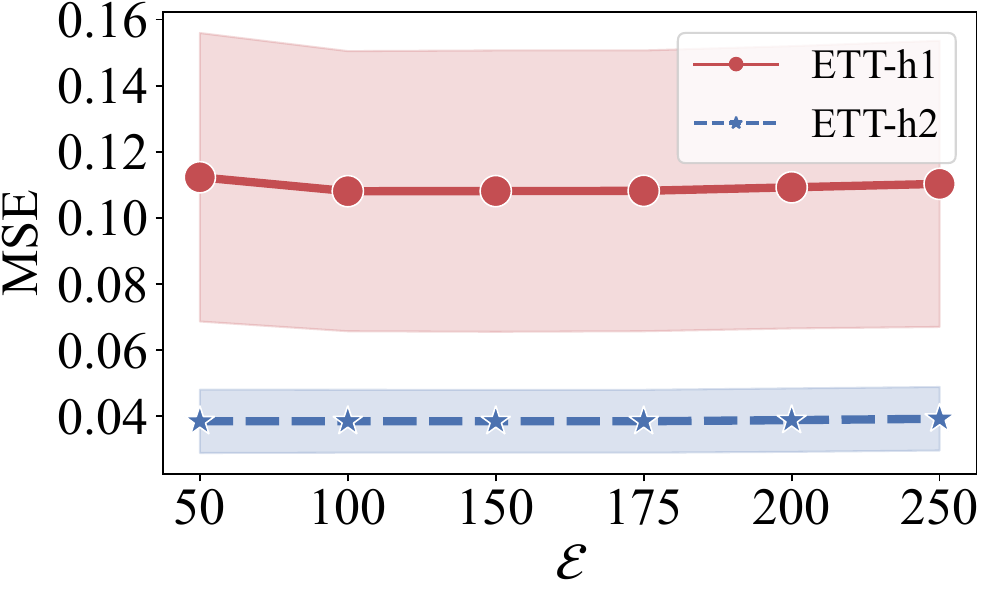}}
      \subfigure[Sensitivity analysis on $\mathrm{H}_{s_\theta}$ .\label{subfig:sens_hidden}]{\includegraphics[width=0.248\linewidth]{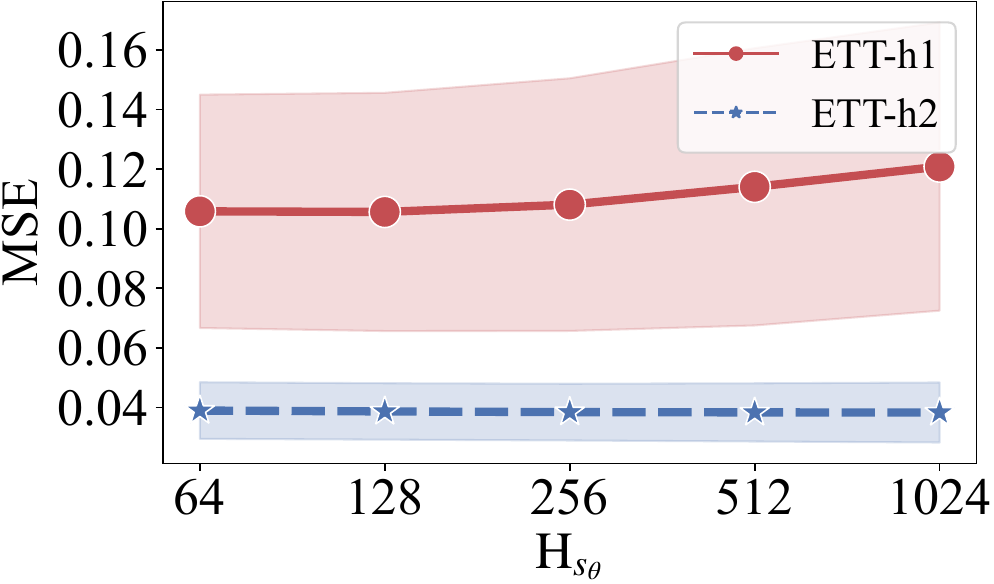}}
     \subfigure[Sensitivity analysis on $T$ .\label{subfig:sens_patch_length}]{\includegraphics[width=0.245\linewidth]{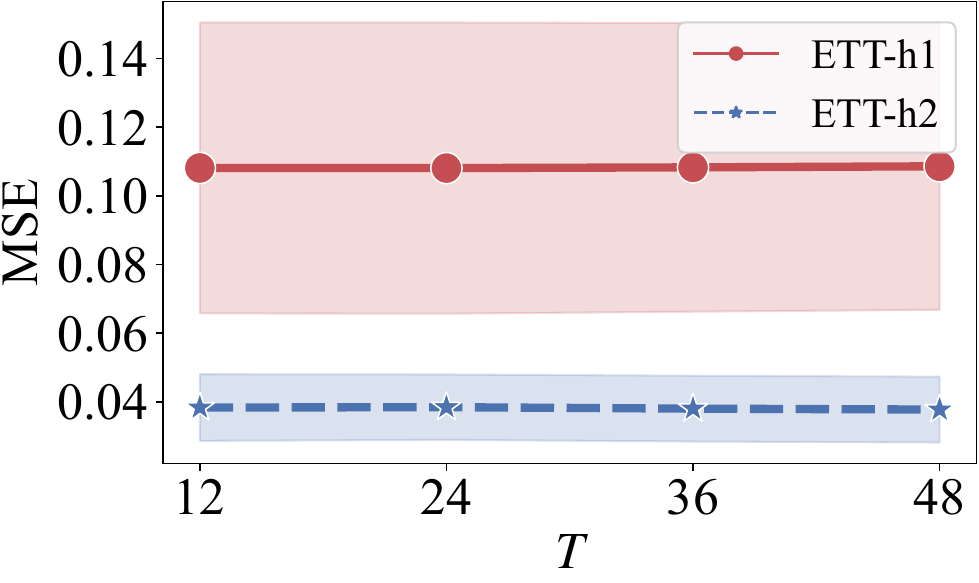}}
    \caption{Sensitivity analysis results on step size: $\eta$, iteration time: $\mathcal{E}$, hidden dimension of score network: $\mathrm{H}_{s_\theta}$, and patch length: $T$. The scatters and shaded areas indicate the mean and one standard deviation from the mean, respectively.}\label{fig:comparisonSensAnalysis}
\vspace{-0.2cm}
\end{figure*}





\subsection{Overall Performance}
\Cref{tab:baselineComparisonAll} presents the overall performance with related baseline models over six missing ratios: 0.1, 0.2, 0.3, 0.4, 0.5, and 0.6. Key observations can be made from~\Cref{tab:baselineComparisonAll}:
\begin{itemize}[leftmargin=*]
\item{\textbf{Efficacy of existing TSDI approaches:} Existing TSDI methods demonstrate strong performance. In particular, discriminative approaches such as PatchTST and TimesNet achieve highly competitive results, ranking first or second in 4 out of 14 cases. These models effectively capture temporal patterns in the data and leverage them for TSDI. Meanwhile, methods specifically designed for TSDI, such as CSDI and PSW-I, also perform comparably, achieving first or second best in 10 out of 14 cases.}
    \item{\textbf{Mass relaxation matters:} Across a range of methods, we observe that vanilla OT-based TSDI approaches, including CSDI, Glocal (from the proximal-term perspective introduced in~\Cref{subsec:analysisBottleNeckResults}), Sinkhorn, and TDM, do not outperform relaxed transportation-based approaches such as PSW-I and SPIRIT. We attribute this phenomenon to the sensitivity of vanilla OT formulations to the non-stationarity inherent in time-series data, which ultimately leads to suboptimal imputation accuracy. 
    }
 \item{\textbf{Necessity of dissipative-structure elimination:} Although CSDI achieves competitive results on most of the datasets, we observe that it can fail in certain datasets, for example, Illness and Traffic. In addition, Glocal and GP-VAE also exhibit consistently weaker performance. This phenomenon suggests that for TSDI, which is an accuracy-oriented task, removing the dissipative-structure term is necessary to ensure reliable performance.
 }

\item{\textbf{Efficacy of the proposed SPIRIT method:} SPIRIT retains the strengths of generative TSDI methods while mitigating their limitations in handling non-stationarity. Overall, SPIRIT achieves the best or second-best performance in 10 out of 14 cases, demonstrating strong effectiveness and practical potential in real-world applications.
    }

    
\end{itemize}

\subsection{Ablation Studies}
In this subsection, we present the ablation study. Our SPIRIT approach has two main contributions: (1) introducing the SPT discrepancy as the proximal term, and (2) removing terms associated with the dissipative structure. Notably, the dissipative structure induced by the VP-SDE is not appropriate in our setting because the conditional distribution defined by the DDPM underlying SDE is incompatible with the SPIRIT framework. We therefore focus on using $\mathbb{E}_{q'}[\log q']$ to represent the dissipative structure. The corresponding results are reported in \Cref{tab:abliationStudyMainResult}.

From \Cref{tab:abliationStudyMainResult}, we observe consistent performance drops when ablating either the SPT proximal term or the $\mathbb{E}_{q'}[\log q']$ component. This indicates that both introducing SPT and eliminating the dissipative-structure term are important to SPIRIT’s overall performance. Moreover, the relative impact differs across datasets. On ETT-h1 and ETT-h2, removing SPT leads to a smaller degradation than removing $\mathbb{E}_{q'}[\log q']$, whereas on ETT-m1 and ETT-m2 the performance drop from removing SPT is larger. This suggests that on larger-scale datasets such as ETT-m1 and ETT-m2, the SPT component contributes more substantially to accuracy improvements, potentially because larger datasets are more likely to contain outliers. This finding highlights the importance of introducing SPT. Moreover, it provides empirical support for~\Cref{prop:robustness} and further motivates the use of the entropy-induced generalized Bregman divergence. Finally, when both components are removed simultaneously, SPIRIT exhibits the most severe degradation, further confirming that SPT and dissipative-structure elimination are complementary and jointly underpin SPIRIT’s superior performance.




\subsection{Sensitivity Analysis}
In this subsection, we conduct sensitivity analysis with-respect-to step size: $\eta$, iteration time: $\mathcal{E}$, hidden dimension of score network: $\mathrm{H}_{s_\theta}$, and patch length: $T$. The corresponding results are proposed in~\Cref{fig:comparisonSensAnalysis}. 

From~\Cref{subfig:sens_eta}, we observe that as the step size $\eta$ increases, SPIRIT’s performance first improves and then degrades. This behavior can be explained by the proximal regularization, whose effective strength is controlled by $\frac{1}{\eta}$. When $\eta$ increases from a small value to a moderate range, $\frac{1}{\eta}$ decreases and the proximal constraint is relaxed, allowing the optimization to focus more on the main objective functional, which improves performance. However, when $\eta$ becomes too large, $\frac{1}{\eta}$ becomes excessively small and the proximal effect is nearly removed, weakening the stabilizing benefits of the proximal operator and leading to performance degradation. From~\Cref{subfig:sens_iter_times}, we observe that as the number of iterations $\mathcal{E}$ increases, SPIRIT’s performance remains nearly unchanged. This suggests that SPIRIT has already converged to a stationary point, indicating a fast convergence rate. Moreover, when increasing the hidden dimension of the score network $\mathrm{H}_{s_\theta}$, \Cref{subfig:sens_hidden} shows a clear performance degradation. This suggests that an overly large $\mathrm{H}_{s_\theta}$ increases the model capacity and makes $s_\theta$ more prone to overfitting, thereby hurting generalization and reducing SPIRIT’s performance Finally, \Cref{subfig:sens_patch_length} indicates that varying the patch length has only a minor effect on performance, demonstrating SPIRIT’s applicability and robustness across different patch-length settings.

In summary, the above sensitivity analysis suggests that, when applying SPIRIT to TSDI, one should use a moderate step size to appropriately balance the proximal regularization and the main objective functional, and adopt a moderate-to-small hidden dimension for the score network to mitigate overfitting issue.

\section{Related Works}

\subsection{DMs for Imputation Task}

DMs have demonstrated strong capabilities in data synthesis~\cite{wangefficiently,10419041}, motivating a growing body of work that adapts them to TSDI~\cite{wang2025optimal}. Existing efforts primarily modify the learning or inference procedures of diffusion models: for instance, Schr\"odinger-bridge formulations have been introduced to accelerate imputation~\citep{chen2023provably}, frequency domain diffusion have been introduced to capture the periodic property~\citep{yang2024frequency}, and mutual-information-based objectives~\citep{liu2024minimizing,yu2025missing} have been used to redesign the training loss and improve performance, as exemplified by Glocal~\citep{glocalImputation}. However, it has been observed that the diversity-seeking nature of diffusion sampling can conflict with the accuracy-oriented objective of imputation~\citep{xu2023density}. To address this tension, \citet{chen2024rethinking} revisited DM-based imputation from a gradient flow viewpoint and proposed a new imputation approach; nonetheless, their analysis is largely framed at the level of a \emph{global} functional and is demonstrated mainly for tabular data, which limits its applicability to time-series settings. In contrast, we characterize an implicitly induced \emph{local} functional that governs the imputation dynamics, which is vital for addressing the non-stationarity and meeting the accuracy requirements of TSDI. As such, we introduce the SPT discrepancy (\Cref{subsec:sptFramework}), re-derive the imputation procedure (\Cref{subsec:diversityEliminatedImputationProcess}), and develop the SPIRIT (\Cref{subsec:overallWorkFlowSPIRIT}); which form our main theoretical contributions.

\subsection{Proximal Recursion for Differential Equation-based Machine Learning Systems}

\citet{jordan1998variational} seminally connected proximal-regularized optimization with the Fokker--Planck equation, providing a proximal-recursion lens for analyzing the associated PDE dynamics. Subsequent work along this line can be broadly grouped into three directions: (i) \textit{Synthesis}, which develops objective functionals and recursive optimization schemes on specific metric spaces (e.g., Wasserstein space) to design new machine learning models for sampling~\citep{neklyudov2023wasserstein,scalableWassersteinUOT}, generation~\citep{JiaoJiaoparGradientFlow,xu2023normalizing}, and density estimation~\citep{8890903,mokrov2021large}; (ii) \textit{Constraint handling}, which incorporates constraints via proximal regularization to solve constrained optimization problems and to enable learning or sampling on restricted domains~\citep{shi2021sampling,hsieh2018mirrored,sharrock2023learning}; and (iii) \textit{Theoretical analysis}, which uses this framework to study existing methods, including convergence and stability~\citep{10583905,fang2025beyond} as well as population-dynamics modeling~\citep{bunne2022proximal,chen2023density}. Our work is most closely related to the \textit{Synthesis} and \textit{Theoretical analysis} lines: we adopt a proximal-recursion viewpoint to examine DMs for TSDI, and develop an alternative formulation tailored to DM-based TSDI.

\section{Conclusions}
In this manuscript, we analyze and improve DM-based TSDI through the lens of proximal-term regularization. We first reformulate the DM-based TSDI procedure as a proximal-recursion process and identify two proximal components that hinder performance: the Wasserstein-distance term and the dissipative-structure term. The former leads to limited robustness under non-stationary features, while the latter tends to encourage diverse imputations rather than accurate reconstructions. To address these issues, we replace the Wasserstein distance with a generalized Bregman divergence induced by an entropy functional, which relaxes the overly restrictive constraint imposed by the Wasserstein metric. We further remove the dissipative-structure term. Based on these modifications, we propose a new DM-based TSDI framework, termed SPIRIT. Extensive experiments validate the effectiveness of SPIRIT.

\section*{Impact Statement}
This paper presents work whose goal is to advance the field of Machine Learning. There are many potential societal consequences of our work, none which we feel must be specifically highlighted here.
\bibliography{ref}
\bibliographystyle{icml2025}

\newpage
\appendix
\onecolumn



    \let\oldsection\section
    \renewcommand{\section}[1]{
        \oldsection{\textcolor{black}{#1}}
        \addcontentsline{apc}{section}{\protect\numberline{\thesection}\textcolor{blue}{#1}}
    }

    \let\oldsubsection\subsection
    \renewcommand{\subsection}[1]{
        \oldsubsection{\textcolor{black}{#1}}
        \addcontentsline{apc}{subsection}{\protect\numberline{\thesubsection}\textcolor{blue}{#1}}
    }
\addcontentsline{toc}{chapter}{Appendices} 

\section{Additional Background Knowledge}\label{sec:additionalBackgroundKnowledge}
In this subsection, we demonstrate the detailed background knowledge we use in the derivation of our main theoretical results in our manuscript.

\paragraph{Differential Equation:}

Suppose we have the following SDE, which is known as the  It\^o process~\citep{oksendal2003stochastic}:
\begin{equation}
    \mathrm{d} \mathbf{x} = 
    f(\mbsx,\tau) 
    \mathrm{d}\tau+ g(\tau)\mathrm{d}W_\tau.
\end{equation}
Denote the marginal distribution of $\mbsx$ at time $\tau$ as $q_\tau(\mbsx)$. It can be observed that the $q_\tau(\mbsx)$ satisfies the following partial differential equation (PDE), which is known as the Fokker-Planck equation~\citep{sarkka2019applied}:
\begin{equation}\label{eq:FPKResult}
 \frac{\partial q_\tau(\mbsx)}{\partial \tau} = -\nabla_\mathbf{x}\cdot[f(\mathbf{x},\tau) q_\tau(\mathbf{x})] + \frac{1}{2} \nabla_\mathbf{x}\cdot [g^2(\tau) \nabla_\mbsx q_\tau(\mbsx)].
\end{equation}
A classical solution to~\Cref{eq:FPKResult} requires $q_\tau(\mbsx)$ to be differentiable in both $\tau$ and $\mbsx$. To address this issue, rather than working with classical (smooth) solutions of~\Cref{eq:FPKResult}, we consider a weak, measure-valued representation of $q_\tau(\mbsx)$ using a finite set of particles $\{\mbsx_{i}\}_{i=1}^{N}$~\citep{liu2017stein}:
\begin{equation}\label{eq:meanWeightParticles}
    q_\tau(\mbsx) \approx \frac{1}{N}\sum_{i=1}^{N}\delta_{\mbsx_{i}}.
\end{equation}
As such, for each particles, we solve the PDE defined by~\Cref{eq:FPKResult} using the following ordinary differential equation (ODE), which merely requires changing the spatial position of $\{\mbsx_{i}\}_{i=1}^{N}$ and provides a weak solution to~\eqref{eq:FPKResult}~\citep{evans2022partial}:
\begin{equation}
    \frac{\dd \mbsx_i }{\dd \tau} = f( \mbsx_i,\tau) - \frac{1}{2}g^2(\tau)\nabla_\mbsx\log q_\tau(\mbsx_i).
\end{equation}

While the standard It\^o process evolves the distribution $q_\tau(\mathbf{x})$ by transporting the particle locations $\{\mathbf{x}_{i}\}_{i=1}^{N}$, an alternative approach is to steer the density by adjusting the particle weights. Specifically, we formulate the empirical approximation of $q_\tau(\mathbf{x})$ as a weighted sum of Dirac measures:
\begin{equation}
       q_\tau(\mathbf{x}) \approx \sum_{i=1}^{N} w_i \delta_{\mathbf{x}_{i}}, \quad\text{s.t.}\quad \underbrace{\sum_{i=1}^{N} w_i = 1, \quad w_i \ge 0}_{[w_1,\ldots,w_{N}]\in \Delta^{N-1}}.
\end{equation}
In this framework, the shape of $q_\tau(\mathbf{x})$ is controlled by the time-varying weights $\{w_i\}_{i=1}^{N}$ while the particle locations remain fixed. The evolution of the probability density is governed by the following integro-differential equation:
\begin{equation} \label{eq:weight_pde}
    \frac{\partial q_\tau(\mathbf{x})}{\partial \tau} = [ \int g_\tau(\mathbf{x}') q_\tau(\mathbf{x}') \mathrm{d} \mathbf{x}' - g_\tau(\mathbf{x}) ] q_\tau(\mathbf{x}),
\end{equation}
where $g_\tau:\mathbb{R}^{D}\to \mathbb{R}$ denotes the scalar function driving the reweighting process. Consequently, the continuous time dynamics for the individual weight $\{w_i\}_{i=1}^{N}$ follow the ODE:
\begin{equation}\label{eq:overallWUpdateResults}
    \frac{\mathrm{d} w_i }{\mathrm{d}\tau} = w_i [  \int g_\tau(\mathbf{x}) q_\tau(\mathbf{x}) \mathrm{d} \mathbf{x} - g_\tau(\mathbf{x}_i) ],\quad \forall i\in\{1,\ldots,N\}.
\end{equation}
Using the fact that $ \dfrac{\mathrm{d} \log w_i }{\mathrm{d}\tau} = \dfrac{1}{w_i} \dfrac{\mathrm{d} w_i }{\mathrm{d}\tau} $,~\Cref{eq:overallWUpdateResults} can be further reformulated as follows:
\begin{equation}\label{eq:overallLogWUpdateResults}
    \frac{\mathrm{d} \log w_i }{\mathrm{d}\tau} = [  \int g_\tau(\mathbf{x}) q_\tau(\mathbf{x}) \mathrm{d} \mathbf{x} - g_\tau(\mathbf{x}_i) ] ,\quad \forall i\in\{1,\ldots,N\}.
\end{equation}
Notably,~\Cref{eq:overallLogWUpdateResults} is the normalized equation for $\{w_i\}_{i=1}^{N}$, and its un-normalized counterpart can be given as follows:
\begin{equation}\label{eq:overallLogWUpdateResultsUnnormalized}
    \frac{\mathrm{d} \log w_i }{\mathrm{d}\tau} = - g_\tau(\mathbf{x}_i )  ,\quad \forall i\in\{1,\ldots,N\}.
\end{equation}
\paragraph{Wasserstein Distance and Its Dynamic Formulation:}
Let $\mathcal{P}_2(\mathbb{R}^{D})$ denote the space of probability measures on $\mathbb{R}^{D}$ with finite second moments, i.e., $\mathcal{P}_2(\mathbb{R}^{D}) \coloneqq \{\mu \in \mathcal{M}(\mathbb{R}^{D}) | \int \| \mbsx \|^2  \mathrm{d}\mu(\mbsx) < \infty \}$, where $\mathcal{M}(\mathbb{R}^{D})$ denotes the set of all probability measures on $\mathbb{R}^{D}$. Considering any two probability measures $\mu, \nu \in \mathcal{P}_2(\mathbb{R}^{D})$, we define the $p$-Wasserstein distance as follows~\citep{villani2009optimal}:
\begin{equation}\label{eq:pWassDefinition}
    \mathbb{W}_{p}^{p}(\mu, \nu) \coloneqq \inf_{\pi\in\Pi(\mu,\nu)} \int_{\mathbb{R}^{D}\times \mathbb{R}^{D}} \| \mbsx-\mbsy \|^p  \mathrm{d}\pi(\mbsx,\mbsy) .
\end{equation}
Here, $\Pi(\mu, \nu)$ represents the collection of all joint distributions supported on $\mathbb{R}^D \times \mathbb{R}^D$ with marginals $\mu$ and $\nu$. The integral formulation corresponds to the Kantorovich optimal transport problem, where the minimizer $\pi^*$ represents the optimal transportation plan.

Even though the Kantorovich optimal transport problem provides a static view, the 2-Wasserstein distance allows for a dynamic interpretation via the \textit{Benamou-Brenier formula}~\citep{ambrosio2021lectures}. Specifically, the Benamou-Brenier formula considers a continuous curve of densities $(q_\tau)_{\tau \in [0,1]}:\mathbb{R}^D\to\mathbb{R}^+$ linking $\mu$ and $\nu$, subject to the law of mass conservation described by the continuity equation as follows~\citep{santambrogio2017euclidean}:
\begin{equation}\label{eq:continuity_eq}
    \frac{\partial q_\tau(\mbsx)}{\partial \tau} + \nabla \cdot [q_\tau(\mbsx) v_\tau(\mbsx)] = 0,
\end{equation}
where $v_\tau: \mathbb{R}^D \to \mathbb{R}^D$ is the velocity field transporting the mass. The squared Wasserstein distance is identified as the minimal kinetic energy required for this transport:
\begin{equation}\label{eq:dynamicOTSolution}
    \mathbb{W}_{2}^2(\mu, \nu) = \inf_{(q_\tau, v_\tau)} \int_0^1 \int_{\mathbb{R}^D} \|v_\tau(\mbsx)\|^2 q_\tau(\mbsx)  \mathrm{d}\mbsx \mathrm{d}\tau,
\end{equation}
where the `$\inf$' operator is taken over all pairs $(q_\tau, v_\tau)$ satisfying Eq.~\eqref{eq:continuity_eq} with boundary conditions $q_0 = \mu$ and $q_1 = \nu$. Based on this, denote the optimal transportation map $\boldsymbol{T}^\star:\mathbb{R}^D\to\mathbb{R}^D$ with an infinitesimal increment $\eta$ as follows:
\begin{equation}
    \boldsymbol{T}^\star(\mbsx) \coloneqq \mbsx + \eta v_\tau^*(\mbsx ),
\end{equation}
we can reformulate the $\mathbb{W}_{2}^2(\mu, \nu)$ based on~\Cref{eq:pWassDefinition,eq:dynamicOTSolution} as follows:
\begin{equation}
      \mathbb{W}_{2}^2(\mu, \nu)  = \int  \Vert \mbsx -   \boldsymbol{T}^\star(\mbsx)  \Vert_2^2 \dd\mu(\mbsx) = \eta^2\int  \Vert v_\tau^*(\mbsx )  \Vert_2^2 \dd\mu(\mbsx). 
\end{equation}


\paragraph{Functional Derivative:} Let $\mathcal{F}:\mathcal{P}_2(\mathbb{R}^{D})\to \mathbb{R}^+$ be a functional over PDF $q:\mathbb{R}^{D}\to \mathbb{R}^+\in \mathcal{P}_2(\mathbb{R}^{D})$. To define the functional derivative, consider a small perturbation $h(\mbsx)$ to the PDF $q(\mbsx)$. The change in the functional $\mathcal{F}[q(\mbsx)]$ can be expressed via the linear expansion with higher order term $o(\|h(\mbsx)\|)$:
\begin{equation}
\begin{aligned}
    \mathcal{F}[q(\mbsx)+ h(\mbsx)] =   \mathcal{F}[q(\mbsx)] + \mathrm{d}  \mathcal{F}[h(\mbsx)] + o(\|h(\mbsx)\|), \quad \text{where} \quad \mathrm{d}  \mathcal{F}[h(\mbsx)] \coloneqq \int   h(\mbsx) \delta_{q(\mbsx)}   \mathcal{F}[q(\mbsx)] \mathrm{d}\mbsx.
\end{aligned}
\end{equation}
Here, $\mathrm{d}\mathcal{F}[h(\mbsx)]$ denotes the first variation of $\mathcal{F}[q(\mbsx)]$ in the direction $h(\mbsx)$, and the function $\delta_{q(\mbsx)}   \mathcal{F}[q(\mbsx)] $ serves as the gradient of the functional with-respect-to the $L^2$ inner product.

\section{Theoretical Derivation}\label{sec:AppendixTheoreticalDerivation}

\paragraph{Organization of~\Cref{sec:AppendixTheoreticalDerivation}.} In this section, we provide detailed derivations for the theoretical results presented in the main manuscript. Specifically, in \Cref{subsec:derivationWhySDEFails}, we prove \Cref{thm:mmsProblemSolving}, which characterizes what DMs implicitly do during the data-imputation process and highlights the issues that motivate this work. Building on this result, in \Cref{subsec:derivationRobustnessProof} we prove \Cref{prop:robustness}, establishing the robustness of the proposed SPT framework to outliers, in contrast to the vanilla OT formulation. We then introduce a novel functional for missing-data imputation and prove \Cref{prop:updatePropositionResults} in \Cref{subsec:derivation4TransTele}. Since the resulting update rule may not be well defined, we further study sufficient conditions to ensure well-posedness in \Cref{subsec:MirrorDescentResults}. Finally, to complete the overall workflow, we derive the learning objective for $\nabla \log p(\mbsxim \mid \mbsxobs)$ in~\Cref{subsec:derivationDSMProp} and present the corresponding convergence analysis in~\Cref{subsec:discussionsOnTheConvergence}.


\subsection{Derivation of~\Cref{thm:mmsProblemSolving} }\label{subsec:derivationWhySDEFails}

\begin{proposition*}[\ref{thm:mmsProblemSolving}]
The imputation process for DMs can be formulated as iteratively solving the following optimization problem in a proximal operator form:
\begin{equation}
\begin{aligned}
\mathop{\inf}_{q'} \indent -\mathbb{E}_{q'}&[\log{p(\mbsx^{\text{imp}}\vert \mbsx^{\text{obs}})}] +\phi(\mbsx^{\text{imp}}) + \frac{1}{\eta} \mathbb{W}_2^2(q', q),
\end{aligned}
\end{equation}
where we abbreviate the candidate distribution  (distribution for current iteration) $q'(\mbsx^{\text{imp}})$ and the base distribution (distribution for previous iteration) $q(\mbsx^{\text{imp}})$ as $q'$ and $q$, respectively. The term $\frac{1}{\eta}$ is a predefined positive term determined by the noise schedule of DMs, and the term $\phi(\mbsx^{\text{imp}}) $ denotes the dissipative structure-related term, which depends on the underlying SDE and is specified as follows:
\begin{itemize}[leftmargin=*]
\item{Variance Preserving SDE (VP-SDE):
$ \phi(\mbsx^{\text{imp}}) =  \frac{1}{2}\mathbb{E}_{q'}[\log{q'(\mbsx^{\text{imp}})}]-\frac{1}{4} \mathbb{E}_{q'}[\Vert \mbsx^{\text{imp}} \Vert_2^2]$, and $\eta=\beta(\tau)$.
}
 \item{Variance Exploding SDE (VE-SDE): $  \phi(\mbsx^{\text{imp}}) =  \frac{1}{2}\mathbb{E}_{q'}[\log{q'(\mbsx^{\text{imp}})}]$, and $\eta=\frac{1}{2}\frac{\dd \sigma^2(\tau)}{\dd \tau}$.
 }
\end{itemize}
\end{proposition*}
\begin{proof}
We begin the considering a more general problem (we abbreviate the time index $t$ in the Fokker-Planck equation given by~\Cref{eq:FPKResult}):
\begin{equation}\label{eq:mmsProblemResult}
    \begin{aligned}
    \mathop{\inf}_{q'(\mbsx)}\quad   \mathcal{F}[q'(\mbsx)] + \frac{1}{2\eta}\mathbb{W}_2^2(q'(\mbsx),q(\mbsx)),
    \end{aligned}
\end{equation}
where $q'(\mbsx)$ is obtained by $q(\mbsx)$ based on the following PDE,:
\begin{equation}
   \frac{\partial q(\mbsx)}{\partial \tau} = -\nabla_\mbsx\cdot[q(\mbsx)v_\tau(\mbsx)] \Rightarrow  q'(\mbsx) = q(\mbsx) -\frac{1}{\eta} \nabla_\mbsx\cdot[q(\mbsx) v_\tau(\mbsx)]+o(\eta^2).
\end{equation}
On this basis, for the 2-Wasserstein distance, we have the following inequality:
\begin{equation}\label{eq:nonOptimalTransportResult}
    \mathbb{W}_2^2(q'(\mbsx),q(\mbsx)) = \int{q(\mbsx) \Vert \mbsx - \boldsymbol{T}^\star(\mbsx)\Vert_2^2\mathrm{d}\mbsx} = \eta^2 \int{q(\mbsx)\Vert  {v}_\tau^*(\mbsx)\Vert_2^2\mathrm{d}\mbsx} \le \eta^2 \int{q(\mbsx)\Vert  {v}_\tau(\mbsx)\Vert_2^2\mathrm{d}\mbsx}  ,
\end{equation}
where ${v}_\tau(\mbsx)$ is the non-optimal transportation velocity field. Meanwhile,~\Cref{eq:mmsProblemResult} can be reformulated as follows:
\begin{equation}\label{eq:mmsProblemResultConstant}
    \begin{aligned}
    \mathop{\inf}_{q'(\mbsx)}\quad   \mathcal{F}[q'(\mbsx)]  -  \underbrace{\mathcal{F}[q(\mbsx)]}_{\text{constant}}+ \frac{1}{2\eta}\mathbb{W}_2^2(q'(\mbsx),q(\mbsx)).
    \end{aligned}
\end{equation}
Thus, we have the following upper bound for~\Cref{eq:mmsProblemResult} based on~\Cref{eq:nonOptimalTransportResult,eq:mmsProblemResultConstant}:
\begin{equation}
\begin{aligned}
       &\mathop{\inf}_{q'(\mbsx)}\quad   \mathcal{F}[q'(\mbsx)]  -  \underbrace{\mathcal{F}[q(\mbsx)]}_{\text{constant}}+ \frac{1}{2\eta}\mathbb{W}_2^2(q'(\mbsx),q(\mbsx))\\
     \overset{\text{(i)}}{  \Rightarrow}  &  \mathop{\inf}_{v_\tau(\mbsx)}\quad   \cancel{\mathcal{F}[q(\mbsx)] }-\eta\int\nabla_\mbsx\cdot [q(\mbsx)v_\tau(\mbsx)]\{\delta_{q(\mbsx)} \mathcal{F}[q(\mbsx)]\}\mathrm{d}\mbsx - \cancel{   \mathcal{F}[q(\mbsx)]}+ \frac{\eta}{2} \int{q(\mbsx)\Vert  {v}_\tau(\mbsx)\Vert_2^2\mathrm{d}\mbsx} \\
  \overset{\text{(ii)}}{  \Rightarrow}  & \mathop{\inf}_{v_\tau(\mbsx)}\quad  \eta\int  v^\top_\tau(\mbsx)\nabla_\mbsx\{\delta_{q(\mbsx)} \mathcal{F}[q(\mbsx)]\}q(\mbsx)\mathrm{d}\mbsx+\frac{\eta}{2}\int{q(\mbsx)\Vert  {v}_\tau(\mbsx)\Vert_2^2\mathrm{d}\mbsx} \\
  \Rightarrow
    & \mathop{\inf}_{v_\tau(\mbsx)}\quad  \eta\int  v^\top_\tau(\mbsx)\nabla_\mbsx \{\delta_{q(\mbsx)} \mathcal{F}[q(\mbsx)]\}q(\mbsx)\mathrm{d}\mbsx+\frac{\eta}{2}\int{q(\mbsx)\Vert  {v}_\tau(\mbsx)\Vert_2^2\mathrm{d}\mbsx} +\underbrace{ \frac{\eta}{2}\int{q(\mbsx ) \Vert \nabla_\mbsx\{\delta_{q(\mbsx)} \mathcal{F}[q(\mbsx)]\} \Vert_2^2\mathrm{d}\mbsx}}_{\ge 0 }\\
   \Rightarrow & \mathop{\inf}_{v_\tau(\mbsx)}\quad  \frac{\eta}{2} \mathbb{E}_{q(\mbsx)}[\Vert v_\tau(\mbsx) + \nabla_\mbsx \{\delta_{q(\mbsx)} \mathcal{F}[q(\mbsx)]\}\Vert_2^2] \\
   \Rightarrow & v_\tau^\star(\mbsx) = - \nabla_\mbsx \{\delta_{q(\mbsx)} \mathcal{F}[q(\mbsx)]\} 
   ,
\end{aligned}
\end{equation}
where ``(i)'' is based on the inequality given by~\Cref{eq:nonOptimalTransportResult}, and ``(ii)'' is based on the integration-by-parts under the mild assumptions~\citep{8744312,dong2022particle}:
\begin{equation}
     \int  v^\top_\tau(\mbsx)\nabla_\mbsx\{\delta_{q(\mbsx)} \mathcal{F}[q(\mbsx)]\}q(\mbsx)\mathrm{d}\mbsx + \int\nabla_\mbsx\cdot [q(\mbsx)v_\tau(\mbsx)]\{\delta_{q(\mbsx)} \mathcal{F}[q(\mbsx)]\}\mathrm{d}\mbsx =\int \nabla_\mbsx\cdot\{\{\delta_{q(\mbsx)} \mathcal{F}[q(\mbsx)]\} q(\mbsx) v_\tau(\mbsx)\} \mathrm{d}\mbsx = 0.
\end{equation}
Based on this, when we have two groups of empirical measures namely $q(\mbsx)$, $q'(\mbsx)$, and the velocity field $v_\tau^\star(\mbsx) $that transports $q(\mbsx)$ to $q'(\mbsx)$ using the following ODE (also can be treated as the weak solution to~\Cref{eq:continuity_eq}):
\begin{equation}\label{eq:}
   \frac{\mathrm{d}\mbsx}{\mathrm{d}\tau} = v_\tau^\star(\mbsx),
\end{equation}
we can treat the velocity filed is obtained by solving the proximal operator with Wasserstein distance as the ``proximal term'' defined by~\Cref{eq:mmsProblemResult}.

Following~\citet{song2020score}, DMs can be broadly categorized based on their underlying stochastic dynamics, specifically the VP and VE SDEs. Despite their differences in noise scheduling, both formulations admit a unified SDE framework, which we leverage to analyze the energy functional during applying them to the TSDI task. On this basis, our analysis for VP and VE SDEs is given as follows:
\begin{itemize}[leftmargin=*]
    \item{\textbf{VP-SDE:} The imputation process, which infers $\mbsxim$ via $\mbsxobs$ by the DMs can be obtained by the following VP-SDE\footnote{We reverse the time axis, thus the coefficient for $\beta(\tau)$ is $1$ rather than $-1$ given by~\citep{song2020score}}:
    \begin{equation}
        \mathrm{d}\mbsxim = [\frac{1}{2}\beta(\tau) \mbsxim  +\beta(\tau) \nabla_\mbsxim\log{p(\mbsxim\vert \mbsxobs)}]\dd \tau  + \sqrt{\beta(\tau)}\dd W_\tau.
    \end{equation}
Based on~\Cref{eq:FPKResult}, the corresponding Fokker-Planck equation that delineates the $q_\tau(\mbsxim)$ can be given as follows:
\begin{equation}
    \frac{\partial q_\tau(\mbsxim)}{\partial \tau} = -\nabla_\mbsxim\cdot [(\frac{1}{2}\beta(\tau)\mbsxim  +\beta(\tau) \nabla_\mbsxim\log{p(\mbsxim\vert \mbsxobs)}) q_\tau(\mbsxim)] + \frac{1}{2}\nabla_\mbsxim\cdot[\beta(\tau)\nabla_\mbsxim q_\tau(\mbsxim) ],
\end{equation}
which is equivalent to the following continuity equation:
\begin{equation}
     \frac{\partial q_\tau(\mbsxim)}{\partial \tau} = -\nabla_\mbsxim\cdot [(\frac{1}{2}\beta(\tau)\mbsxim  +\beta(\tau) \nabla_\mbsxim\log{p(\mbsxim\vert \mbsxobs)}- \frac{1}{2}\nabla_\mbsxim\log{q_\tau(\mbsxim)}) q_\tau(\mbsxim) ] .
\end{equation}
Based on this, the corresponding velocity filed $v^\star_\tau(\mbsxim)$ can be given as follows:
\begin{equation}
   v^\star_\tau(\mbsxim) = \frac{1}{2}\beta(\tau)\mbsxim  +\beta(\tau) \nabla_\mbsxim\log{p(\mbsxim\vert \mbsxobs)}- \frac{1}{2}\beta(\tau)\nabla_\mbsxim\log{q_\tau(\mbsxim)} .
\end{equation}
Consequently, the corresponding optimization problem can be given as follows:
\begin{equation}\label{eq:ouSDEProximalResult}
\begin{aligned}
\mathop{\inf}_{q'(\mbsxim)}\indent  -\mathbb{E}_{q'(\mbsxim)}[\log{p(\mbsxim\vert \mbsxobs)}]+
\overbrace{\frac{1}{2}\mathbb{E}_{q'(\mbsxim)}[\log{q'(\mbsxim)}]-\frac{1}{4}\mathbb{E}_{q'(\mbsxim)}[\Vert \mbsxim\Vert_2^2] }^{ \phi(\mbsxim)}
+ \frac{1}{2\beta(\tau)} \mathbb{W}_2^2(q'(\mbsxim), q(\mbsxim)),
\end{aligned}
\end{equation}
    }
 \item{\textbf{VE-SDE:} The imputation process, which infers $\mbsxim$ via $\mbsxobs$ by the DMs can be obtained by the following VE-SDE\footnote{We reverse the time axis, thus the $\frac{\mathrm{d}\sigma^2(\tau)}{\dd \tau}$ is $1$ rather than $-1$ given by~\citep{song2020score}}:
    \begin{equation}
        \mathrm{d}\mbsxim =\frac{\dd \sigma^2(\tau)}{\dd \tau} \nabla_\mbsxim\log{p(\mbsxim\vert \mbsxobs)}\dd \tau +   \sqrt{\frac{\dd \sigma^2(\tau)}{\dd \tau}}\dd W_\tau.
    \end{equation}
Based on~\Cref{eq:FPKResult}, the corresponding Fokker-Planck equation that delineates the $q_\tau(\mbsxim)$ can be given as follows:
\begin{equation}
    \frac{\partial q_\tau(\mbsxim)}{\partial \tau} = -\nabla_\mbsxim \cdot [(\frac{\dd \sigma^2(\tau)}{\dd \tau} \nabla_\mbsxim\log{p(\mbsxim\vert \mbsxobs)}) q_\tau(\mbsxim) ] +   \frac{1}{2}\nabla_\mbsxim\cdot[\frac{\dd \sigma^2(\tau)}{\dd \tau} \nabla_\mbsxim q_\tau(\mbsxim) ],
\end{equation}
which is equivalent to the following continuity equation:
\begin{equation}
    \frac{\partial q_\tau(\mbsxim)}{\partial \tau} = - \nabla_\mbsxim\cdot[(\frac{\dd \sigma^2(\tau)}{\dd \tau} \nabla_\mbsxim\log{p(\mbsxim\vert \mbsxobs)}-\frac{1}{2}\frac{\dd \sigma^2(\tau)}{\dd \tau}\nabla_\mbsxim \log{q_\tau(\mbsxim)} ) q_\tau(\nabla_\mbsxim)].
\end{equation}
Based on this, the corresponding velocity filed $v^\star_\tau(\mbsxim)$ can be given as follows:
\begin{equation}
   v^\star_\tau(\mbsxim) = \frac{\dd \sigma^2(\tau)}{\dd \tau} \nabla_\mbsxim\log{p(\mbsxim\vert \mbsxobs)}-\frac{1}{2}\frac{\dd \sigma^2(\tau)}{\dd \tau}\nabla_\mbsxim \log{q_\tau(\mbsxim)}  .
\end{equation}
Consequently, the corresponding optimization problem can be given as follows:
\begin{equation}\label{eq:langSDEProximalResult}
\begin{aligned}
\mathop{\inf}_{q'(\mbsxim)}\indent  -\mathbb{E}_{q'(\mbsxim)}[\log{p(\mbsxim\vert \mbsxobs)}]+
\overbrace{\frac{1}{2}\mathbb{E}_{q'(\mbsxim)}[\log{q'(\mbsxim)}] }^{ \phi(\mbsxim)}
+ \frac{1}{\frac{\dd \sigma^2(\tau)}{\dd \tau}} \mathbb{W}_2^2(q'(\mbsxim), q(\mbsxim)).
\end{aligned}
\end{equation}
 }
\end{itemize}
By observing~\Cref{eq:ouSDEProximalResult,eq:langSDEProximalResult}, we arrive at the desired results.
\end{proof}
\subsection{Derivation of~\Cref{prop:robustness}}
\begin{lemma*}[\ref{prop:robustness}]\label{subsec:derivationRobustnessProof}
Let $\mu$ and $\nu$ be probability measures on $\mathbb{R}^D$, and consider a contaminated target distribution $\tilde{\nu} = (1-\zeta)\nu + \zeta\delta_{\mathbf{z}}, \zeta\in(0,1)$, where $\delta_\mathbf{z}$ denotes a Dirac mass at the outlier location $\mathbf{z} \in\mathbb{R}^D$. The Wasserstein distance has the following lower bound:
\begin{equation}\label{eq:wassDistLower}
\begin{aligned}
     \mathbb{W}_2^2&(\mu,\tilde{\nu}) \ge \zeta\mathbb{W}_2^2(\mu,{\nu})   +(1-\zeta) [\|\mathbf{y}^*-\mathbf{z}\|_2^2 -g(\mathbf{y}^*)+ \int{h(\mathbf{x})\mu(\mathbf{x})\dd\mathbf{x}}],
\end{aligned}
\end{equation}
for some $\mathbf{y}^*$ belonging to the support of $\nu$, and where $f$ and $g$ are optimal dual potentials for $\mathbb{W}_2^2(\mu,\nu)$. Meanwhile, the SPT discrepancy with Bregman potential $\psi(\nu) \coloneqq  \int \nu(\mathbf{y})[\log \nu(\mathbf{y}) - 1] \mathrm{d}\mathbf{y}$ under target contamination admits the bound as follows:
\begin{equation}\label{eq:semiProxDistUpper}
\mathbb{S}(\mu,\tilde{\nu})\le(1-\zeta)\mathbb{S}(\mu,\nu)+\zeta(1-e^{-D(\mathbf{z}) }) + C(\zeta),
\end{equation}
where $D(\mathbf{z})\coloneqq \int\| \mathbf{z} - \mathbf{x} \|_2^2 \mu(\mathbf{x})\dd\mathbf{x} $ is the average distance of $\mathbf{z}$ and samples from $\nu$, and $C(\zeta)$ is a constant defined as $C(\zeta)\coloneqq (1-\zeta)\log\frac{1}{1-\zeta}-\zeta\log\zeta$.
\end{lemma*}

\begin{proof}
Our proof is based on previous works~\citep{fatras2021unbalanced,wang2025optimal}, and is divided into two parts, namely the derivation of the lower bound given by~\Cref{eq:wassDistLower} and the derivation of the upper bound given by~\Cref{eq:semiProxDistUpper}. 

\paragraph{Derivation of~\Cref{eq:wassDistLower}:} Let $C(\mathbf{x},\mathbf{y})=\|\mathbf{x}-\mathbf{y}\|_2^2$. The Kantorovich dual for $\mathbb{W}_2^2(\mu,\nu)$ reads
\begin{equation}
\mathbb{W}_2^2(\mu,\nu)=\sup_{h,g:\ h(\mathbf{x})+g(\mathbf{y})\le C(\mathbf{x},\mathbf{y})}\{\int h(\mathbf{x})\dd\mu(\mathbf{x})+\int g(\mathbf{y})\dd\nu(\mathbf{y})\},
\end{equation}
where the $h$ and $g$ are called optimal dual potential. Let $(h,g)$ be an optimal dual pair for $(\mu,\nu)$. Consider the contaminated target as follows:
\begin{equation}
\tilde\nu=(1-\zeta)\nu+\zeta\delta_{\mathbf{z}},\qquad \zeta\in(0,1),
\end{equation}
with $\mathbf{z}\in\mathbb{R}^D$. Define a new potential $\tilde g$ by keeping $g$ on the support $\operatorname{supp}(\nu)$ and extending it at $\mathbf{z}$ via the $c$-transform:
\begin{equation}
\tilde g(\mathbf{y})=g(\mathbf{y})\ \text{for }\mathbf{y}\in\operatorname{supp}(\nu),\qquad 
\tilde g(z)=h^{c}(\mathbf{z}):=\inf_{\mathbf{y}\in\mathbb{R}^D}\bigl(c(\mathbf{z},\mathbf{y})-g(\mathbf{y})\bigr).
\end{equation}
By construction of the $c$-transform, we have $h(\mathbf{x})+\tilde g(\mathbf{z})\le c(\mathbf{x},\mathbf{z})$ for all $\mathbf{x}$, and on $\operatorname{supp}(\nu)$ we have $h(\mathbf{x})+\tilde g(\mathbf{y})=h(\mathbf{x})+g(\mathbf{y})\le C(\mathbf{x},\mathbf{y})$. Hence $(h,\tilde g)$ is feasible for the dual problem associated with $(\mu,\tilde\nu)$. Therefore,
\begin{equation}\label{eq:theLowerBoundForWassDist}
\begin{aligned}
\mathbb{W}_2^2(\mu,\tilde\nu)
&=\sup_{h',g': h'(\mathbf{x})+g'(\mathbf{y})\le C(\mathbf{x},\mathbf{y})}\{\int h'(\mathbf{x})\dd\mu(\mathbf{x})+\int g'(\mathbf{y})\dd\tilde\nu(\mathbf{y})\}\\
&\ge \int h(\mathbf{x})\dd\mu(\mathbf{x})+\int \tilde g(\mathbf{y})\dd\tilde\nu(\mathbf{y})\\
&= \int h(\mathbf{x})\dd\mu(\mathbf{x})+(1-\zeta)\int g(\mathbf{y})\dd\nu(\mathbf{y})+\zeta\tilde g(\mathbf{z})\\
&= (1-\zeta)[\int h(\mathbf{x})\dd\mu(\mathbf{x})+\int g(\mathbf{y})\dd\nu(\mathbf{y})]+\zeta[\tilde g(\mathbf{z})+\int h(\mathbf{x})\dd\mu(\mathbf{x})]\\
&= (1-\zeta)\mathbb{W}_2^2(\mu,\nu)+\zeta[h^{c}(\mathbf{z})+\int h(\mathbf{x})\dd\mu(\mathbf{x})],
\end{aligned}
\end{equation}
where the last equality uses optimality of $(h,g)$ for $(\mu,\nu)$, i.e. $\int h\dd\mu+\int g\dd\nu=\mathbb{W}_2^2(\mu,\nu)$.

Finally, 
the infimum in $h^{c}(\mathbf{z})$ is achieved at some $\mathbf{y}^*$ from support $\nu$, yielding
\begin{equation}\label{eq:infForHc}
h^{c}(\mathbf{z})=\inf_{\mathbf{y}}\|\mathbf{z}-\mathbf{y}\|_2^2-g(\mathbf{y}))=\|\mathbf{z}-\mathbf{y}^*\|_2^2-g(\mathbf{y}^*).
\end{equation}
Plugging~\Cref{eq:infForHc} into~\Cref{eq:theLowerBoundForWassDist}, we get the following result:
\begin{equation}
\mathbb{W}_2^2(\mu,\tilde\nu)\ \ge\ (1-\zeta)\mathbb{W}_2^2(\mu,\nu)\ +\ \zeta[\|\mathbf{z}-\mathbf{y}^*\|_2^2-g(\mathbf{y}^*)+\int h(\mathbf{x})d\mu(\mathbf{x})].
\end{equation}
This proves the claimed lower bound.

\paragraph{Derivation of~\Cref{eq:semiProxDistUpper}:} When we set the Bregman potential as $\psi(\nu) \coloneqq  \int \nu(\mathbf{y})[\log \nu(\mathbf{y}) - 1] \mathrm{d}\mathbf{y}$, the SPT discrepancy can be reformulated as follows:
\begin{equation}\label{eq:expotentialFunctionBregSPT}
\begin{aligned}
     & \mathbb{S}(\mu, \nu) \\
    =&  \inf_{\pi \in \Pi(\mu)}
    \int \|\mathbf{x}-\mathbf{y}\|^2 \mathrm{d}\pi(\mathbf{x},\mathbf{y}) 
    + \psi(\pi_{\mathbf{y}}) - \psi(\nu) - \langle \delta_\nu \psi(\nu), \pi_{\mathbf{y}} - \nu \rangle \\
    = & \inf_{\pi \in \Pi(\mu)}
    \int \|\mathbf{x}-\mathbf{y}\|^2 \mathrm{d}\pi(\mathbf{x},\mathbf{y}) 
    + \int \pi_\mathbf{y}(\mathbf{y} )[\log\pi_\mathbf{y}(\mathbf{y} ) - 1] \dd\mathbf{y} 
     - \int \nu(\mathbf{y})[\log\nu(\mathbf{y}) - 1] \dd\mathbf{y} - \int \log\nu(\mathbf{y})[ \pi_\mathbf{y}(\mathbf{y} )- \nu(\mathbf{y}) ] \dd \mathbf{y}\\
    = &  \inf_{\pi \in \Pi(\mu)}
    \int \|\mathbf{x}-\mathbf{y}\|^2 \mathrm{d}\pi(\mathbf{x},\mathbf{y}) 
    + \int \pi_\mathbf{y}(\mathbf{y} )\log{\frac{\pi_\mathbf{y}(\mathbf{y} )}{\nu(\mathbf{y})}} \dd\mathbf{y} .
\end{aligned}
\end{equation}

On this basis, we consider a scalable contamination model by replacing $\tilde\nu=(1-\zeta)\nu+\zeta\delta_z$ with $\tilde\nu_\varphi=(1-\zeta)\nu+\zeta\varphi\delta_z$, where $\varphi\ge0$ is an optimizable parameter. Let $\pi^\star\in\Pi(\mu)$ be an optimal plan for $\mathbb S(\mu,\nu)$, i.e.,
\begin{equation}
\mathbb S(\mu,\nu)=\int \|\mathbf{x}-\mathbf{y}\|^2d\pi^\star(\mathbf{x},\mathbf{y})+D_\psi(\pi_y^\star,\nu),
\end{equation}
so that $\widehat\pi_{\mathbf{y}}=(1-\zeta)\pi^\star_{\mathbf{y}}+\zeta\delta_{\mathbf{z}}$. Let us fix any $\varphi\ge 0$ and set $\tilde\nu_\varphi\coloneqq (1-\zeta)\nu+\zeta\varphi\delta_{\mathbf{z}}$. 

By linearity of the cost, we have the following result:
\begin{equation}
\int \|{\mathbf{x}}-{\mathbf{y}}\|^2_2 \dd\widehat\pi({\mathbf{x}}, {\mathbf{y}})
=(1-\zeta)\int \|{\mathbf{x}}-{\mathbf{y}}\|^2\dd\pi^\star({\mathbf{x}}, {\mathbf{y}})+\zeta\int \|{\mathbf{x}} - {\mathbf{z}}\|^2 \mu(\dd \mathbf{x}).
\end{equation}

Using the convexity of $p\mapsto D_\psi(p,q)$ in its first argument for fixed $q$, we have
\begin{equation}\label{eq:inequalityUpperBound1_new}
D_\psi(\widehat\pi_{\mathbf y},\tilde\nu_\varphi)
\le (1-\zeta)D_\psi(\pi^\star_{\mathbf y},\tilde\nu_\varphi)
+\zeta D_\psi(\delta_{\mathbf z},\tilde\nu_\varphi).
\end{equation}
Moreover, since $\mathbf{z}$ is outlier, it is justified to introduce the assumption that $\nu(\{\mathbf z\})\approx0$. Hence, we have:
\begin{equation}
D_\psi(\delta_{\mathbf z},\tilde\nu_\varphi)
=-\log(\zeta\varphi).
\end{equation}

Since $\tilde\nu_\varphi \ge (1-\zeta)\nu$ as measures, we obtain the following result when $\pi^\star_{\mathbf y}\ll\nu$ (to promise the well-definess for the computation of Bregman divergence):
\begin{equation}\label{eq:inequalityUpperBound2_new}
D_\psi(\pi^\star_{\mathbf y},\tilde\nu_\varphi)
=\int \pi^\star_{\mathbf y}(\mathbf y)\log\frac{\pi^\star_{\mathbf y}(\mathbf y)}{\tilde\nu_\varphi(\mathbf y)}
\le \int \pi^\star_{\mathbf y}\log\frac{\pi^\star_{\mathbf y}(\mathbf y)}{(1-\zeta)\nu(\mathbf y)}
= D_\psi(\pi^\star_{\mathbf y},\nu)+\log\frac1{1-\zeta}.
\end{equation}

Collecting terms and using the feasible coupling $\widehat\pi$, we obtain, for any $\varphi\ge0$, we get the following inequality:
\begin{equation}\label{eq:main_upper_new}
\mathbb S(\mu,\tilde\nu)
\le (1-\zeta)\mathbb S(\mu,\nu)
+\zeta[\varphi D(\mathbf{z})+(\varphi\log\varphi-\varphi+1)]
+ C(\zeta),
\end{equation}
where $D(\mathbf{z}):=\int\|\mathbf{x}-\mathbf{z}\|^2\mu(\dd\mathbf{x})$, and
$C(\zeta)\coloneqq (1-\zeta)\log\frac{1}{1-\zeta}-\zeta\log\zeta$ is a constant (for fixed $\zeta$) independent of $\varphi$.

Finally, since \Cref{eq:main_upper_new} holds for any $\varphi\ge0$, we minimize the $\varphi$-dependent term.
Define
\begin{equation}
e(\varphi):=\varphi D(\mathbf{z})+(\varphi\log\varphi-\varphi+1),
\end{equation}
Then take the first-order condition for the optimization problem, we have:
\begin{equation}
\frac{\dd e(\varphi)}{\dd \varphi}=D(\mathbf{z})+\log\varphi=0\Rightarrow \varphi^\star  =\exp(-D(\mathbf{z})).
\end{equation}
On this basis, consider the second-order condition, we have:
\begin{equation}
\frac{\dd^2 e(\varphi)}{\dd \varphi^2}=\frac{1}{\varphi}> 0.
\end{equation}
Thus, we have the following result:
\begin{equation}\label{eq:infOfVarPhiResult}
    \inf_{\varphi\ge 0} \quad  e(\varphi)  = D(z)e^{-D(z)}  - d(z)e^{-D(z)} - e^{-D(z)} + 1 =  1-e^{-D(z)}.
\end{equation}
Plugging \Cref{eq:infOfVarPhiResult} into \Cref{eq:main_upper_new} yields the desired bound.
\end{proof}

\subsection{Derivation of Proposition~\ref{prop:updatePropositionResults}}\label{subsec:derivation4TransTele}
\begin{proposition*}[\ref{prop:updatePropositionResults}]
Assume $p(\mbsx^{\text{imp}}| \mbsx^{\text{obs}})\in C^1$ and $\nabla \log p(\mbsx^{\text{imp}}| \mbsx^{\text{obs}})$ is square-integrable under the measures considered. Let $\psi(\rho)=\int \rho(\mbsx)[\log\rho(\mbsx)-1]\mathrm d\mbsx$. Represent $q'$ by an empirical measure $q'=\sum_{i=1}^N w_i\delta_{\mbsx_i}$ with $w_i\ge 0$ and $\sum_{i=1}^N w_i=1$. Then, the descent directions $\boldsymbol{T}$ for updating the locations $\{\mbsx_i\}_{i=1}^N$ and weights $\{w_i\}_{i=1}^N$, which yield an approximate solution to \Cref{eq:lossFuncDensityEliminated}, are given as follows:
\begin{itemize}[leftmargin=*]
    \item \textbf{Location direction}, where $\{\mbsx_i\}_{i=1}^N$ are updated by:
    \begin{equation}
        \boldsymbol{T}_{\mbsxim}(\mbsxim)
        = \nabla \log p(\mbsx^{\mathrm{imp}}| \mbsx^{\mathrm{obs}}).
    \end{equation}
    \item \textbf{Weight direction}, where $\{w_i\}_{i=1}^N$ are updated by:
\begin{equation}
\begin{aligned}
    \boldsymbol{T}_{w} =- 2 &\|\nabla\log p(\mbsx^{\text{imp}}\vert \mbsx^{\text{obs}})\|_2^2  + 2\mathbb{E}[\| \nabla\log p(\mbsx^{\text{imp}}\vert \mbsx^{\text{obs}})\|_2^2].
\end{aligned}
\end{equation}    
\end{itemize}
\end{proposition*}

\begin{proof}
At the beginning, we start handling the SPT discrepancy $\mathbb{S}$. Specifically, for $\mathbb{S}(\mu, \nu) $, when introducing an intermediate marginal distribution $\sigma \coloneqq  \pi_{\mathbf{y}}$ we can disintegrate the feasible set according to the $\mathbf{y}$-marginal:

\begin{equation}
\begin{aligned}
   &  \mathbb{S}(\mu, \nu) 
    =  \inf_{\pi \in \Pi(\mu)}
    \int \|\mathbf{x}-\mathbf{y}\|^2 \mathrm{d}\pi(\mathbf{x},\mathbf{y}) 
    + D_\psi(\pi_\mathbf{y}, \nu)
    \\
    \Rightarrow &
      \mathbb{S}(\mu, \nu)  = \inf_{\sigma}\ \inf_{\pi:\ \pi_{\mathbf{x}}=\mu,\ \pi_{\mathbf{y}}=\sigma}
\int \|\mathbf{x}-\mathbf{y}\|^2  \mathrm{d}\pi(\mathbf{x},\mathbf{y})
+ D_\psi(\pi_\mathbf{y}, \nu)
 \\
\Rightarrow & \mathbb{S}(\mu, \nu)   = \inf_{\sigma}\ \inf_{\pi:\ \pi_{\mathbf{x}}=\mu,\ \pi_{\mathbf{y}}=\sigma}
\int \|\mathbf{x}-\mathbf{y}\|^2  \mathrm{d}\pi(\mathbf{x},\mathbf{y})
+D_\psi(\sigma, \nu)
 \\
\Rightarrow  & \mathbb{S}(\mu, \nu)  = \inf_{\sigma}
\Bigg\{
\underbrace{\inf_{\pi \in \Pi(\mu,\sigma)}
\int \|\mathbf{x}-\mathbf{y}\|^2  \mathrm{d}\pi(\mathbf{x},\mathbf{y})}_{=\mathbb{W}_2^2(\mu,\sigma)}+
D_\psi(\sigma, \nu)
\Bigg\} \\
\Rightarrow & \mathbb{S}(\mu, \nu)  = \inf_{\sigma} \mathbb{W}_2^2(\mu,\sigma) + D_\psi(\sigma, \nu)
,
\end{aligned}
\end{equation}

On this basis, when introducing the intermediate 
$\widehat{q}(\mbsxim)$, we have the following result:
\begin{equation}
\begin{aligned}
 &  \mathop{\inf}_{q'} \indent -\mathbb{E}_{q'}[\log  p(\mbsx^{\text{imp}}\vert  \mbsx^{\text{obs}})]  + \frac{1}{2} \mathbb{E}_{q'}[\| \nabla\log{p(\mbsx^{\text{imp}}\vert \mbsx^{\text{obs}})}  \|_2^2 ] + \frac{1}{2\eta } \big\{ \inf_{\widehat{q}} \mathbb{W}_2^2(q',\widehat{q}) + D_\psi(\widehat{q}, q') \big\} \\
\Rightarrow  & \mathop{\inf}_{q', \widehat{q}} \indent -\mathbb{E}_{q'}[\log  p(\mbsx^{\text{imp}}\vert  \mbsx^{\text{obs}})]  + \frac{1}{2} \mathbb{E}_{q'}[\| \nabla\log{p(\mbsx^{\text{imp}}\vert \mbsx^{\text{obs}})}  \|_2^2 ] + \frac{1}{2\eta } [ \mathbb{W}_2^2(q',\widehat{q}) + D_\psi(\widehat{q}, q') ].
\end{aligned}
\end{equation}

To decouple location and weight updates, innovated by the Cole-Hopf transformation~\citep{8890903,chen2021likelihood,caluya2021wasserstein} and alternating direction method of multipliers-based optimization methods~\citep{lin2022alternating}, we introduce two split measures for $q'$ as follows:
\begin{equation}
    \begin{cases}
\tilde q = \sum_{i=1}^N w_i^{(k)} \delta_{\mathbf{x}_i^{\text{imp}}}\\
\widehat q = \sum_{i=1}^N \widehat w_i \delta_{\mathbf{x}_i^{\text{imp}}}^{(k+1)}
    \end{cases},
\end{equation}
where $\tilde q$ updates \emph{locations} $\{\mathbf{x}_i^{\text{imp}}\}_{i=1}^N$ with fixed weights $\{w_i^{(k)}\}_{i=1}^{N}$, and
$\widehat q$ updates \emph{weights} $\{\widehat w_i\}$ on the fixed support $\{ \delta_{\mathbf{x}_i^{\text{imp}}}^{(k+1)} \}_{i=1}^N$.
Accordingly, we introduce the following definition:
\begin{equation}
    \begin{cases}
\tilde q \in 
\tilde{\mathcal{Q}}
\coloneqq  \{ \sum_{i=1}^N w_i^{(k)} \delta_{\mathbf{x}_i^{\text{imp}}} \}\\
\widehat q \in 
\widehat{\mathcal{Q}}
\coloneqq \{ \sum_{i=1}^N w_i \delta_{\mathbf{x}_i^{\text{imp}}}^{(k+1)}: w_i \ge 0,\sum_{i=1}^{N} w_i=1 \} 
    \end{cases}.
\end{equation}
We then adopt an alternating splitting scheme. Given $q'^{(k)}$, we obtain the following surrogate objective functional:
\begin{equation}
    \mathop{\inf}_{\tilde{q}\in\tilde{\mathcal{Q}}, \widehat{q}\in\widehat{\mathcal{Q}}} \indent -\mathbb{E}_{\tilde{q}}[\log  p(\mbsx^{\text{imp}}\vert  \mbsx^{\text{obs}})]  + \frac{1}{2} \mathbb{E}_{\widehat{q}}[\| \nabla\log{p(\mbsx^{\text{imp}}\vert \mbsx^{\text{obs}})}  \|_2^2 ] + \frac{1}{2\eta } [ \mathbb{W}_2^2(\tilde{q},q) + D_\psi(\widehat{q}, q') ]
\end{equation}

On this basis, we can perform the following optimization process recursively:
\begin{itemize}[leftmargin=*]
    \item{\textbf{Transportation Step:} In this step, we change the spatial location of samples $\{\mathbf{x}_i^{\text{imp}}\}_{i=1}^N$ by solving the following optimization problem:
    \begin{equation}
\tilde q^{(k+1)} \in \inf_{\tilde q \in \tilde{\mathcal{Q}}
}
\left\{
-\mathbb{E}_{\tilde q}[\log p(\mbsx^{\text{imp}}\vert \mbsx^{\text{obs}})]
+\frac{1}{2\eta}\mathbb W_2^2(\tilde q, q'^{(k)})
\right\}.
\end{equation}
 we have the following upper bound for~\Cref{eq:mmsProblemResult} based on~\Cref{eq:nonOptimalTransportResult,eq:mmsProblemResultConstant}:
\begin{equation}\label{eq:transportVeloctiyFieldImputation}
\begin{aligned}
       &\mathop{\inf}_{\tilde{q}}\quad  -\mathbb{E}_{\tilde q}[\log p(\mbsx^{\text{imp}}\vert \mbsx^{\text{obs}})]  + \mathbb{E}_{ q'^{(k)}}[\log p(\mbsx^{\text{imp}}\vert \mbsx^{\text{obs}})] + \frac{1}{2\eta}\mathbb{W}_2^2(\tilde{q},q'^{(k)})\\
     \overset{\text{(i)}}{  \Rightarrow}  &  \mathop{\inf}_{v_\tau(\mbsxim)}\quad   \cancel{ -\mathbb{E}_{ q'^{(k)}}[\log p(\mbsx^{\text{imp}}\vert \mbsx^{\text{obs}})]  }+\eta\int\nabla_\mbsxim\cdot [q'^{(k)}(\mbsxim)v_\tau(\mbsxim)][\log p(\mbsx^{\text{imp}}\vert \mbsx^{\text{obs}})]\mathrm{d}\mbsxim  \\
     & \quad \quad \quad \quad\quad \quad \quad \quad + \cancel{    \mathbb{E}_{ q'^{(k)}}[\log p(\mbsx^{\text{imp}}\vert \mbsx^{\text{obs}})] } + \frac{\eta}{2} \int{q'^{(k)}(\mbsxim)\Vert  v_\tau(\mbsxim)\Vert_2^2\mathrm{d}\mbsxim} \\
  \overset{\text{(ii)}}{  \Rightarrow}  & \mathop{\inf}_{v_\tau(\mbsxim)}\quad  -\eta\int  v^\top_\tau(\mbsxim)\nabla_\mbsxim\log p(\mbsx^{\text{imp}}\vert \mbsx^{\text{obs}}) q'^{(k)}(\mbsxim)\mathrm{d}\mbsxim+\frac{\eta}{2}\int{q'^{(k)}(\mbsxim)\Vert  {v}_\tau(\mbsxim)\Vert_2^2\mathrm{d}\mbsxim} \\
  \Rightarrow
    & 
    \mathop{\inf}_{v_\tau(\mbsxim)}\quad  -\eta\int  v^\top_\tau(\mbsxim)\nabla_\mbsxim\log p(\mbsx^{\text{imp}}\vert \mbsx^{\text{obs}}) q'^{(k)}(\mbsxim)\mathrm{d}\mbsxim+\frac{\eta}{2}\int{q'^{(k)}(\mbsxim)\Vert  {v}_\tau(\mbsxim)\Vert_2^2\mathrm{d}\mbsxim}  \\
    & \quad \quad \quad \quad\quad \quad \quad \quad  +\underbrace{ \frac{\eta}{2}\int{q'^{(k)}(\mbsxim ) \Vert \nabla_\mbsxim\log p(\mbsx^{\text{imp}}\vert \mbsx^{\text{obs}}) \Vert_2^2\mathrm{d}\mbsx}}_{\ge 0 }
    \\
   \Rightarrow & \mathop{\inf}_{v_\tau(\mbsxim)}\quad  \frac{\eta}{2} \mathbb{E}_{q'^{(k)}(\mbsxim)}[\Vert v_\tau(\mbsxim) -\nabla_\mbsxim\log p(\mbsx^{\text{imp}}\vert \mbsx^{\text{obs}}) \Vert_2^2] \\
   \Rightarrow & v_\tau^\star(\mbsxim) = \nabla_\mbsxim\log p(\mbsx^{\text{imp}}\vert \mbsx^{\text{obs}})  
   ,
\end{aligned}
\end{equation}
where ``(i)'' is based on the inequality given by~\Cref{eq:nonOptimalTransportResult}, and ``(ii)'' is based on the integration-by-parts under the mild assumptions~\citep{8744312,dong2022particle}:
\begin{equation}
\begin{aligned}
    & \int  v^\top_\tau(\mbsxim)\nabla_\mbsxim\log p(\mbsx^{\text{imp}}\vert \mbsx^{\text{obs}})q'^{(k)}(\mbsxim )\mathrm{d}\mbsxim   + \int\nabla_\mbsxim\cdot [q'^{(k)}(\mbsxim )v_\tau(\mbsxim)]\log p(\mbsx^{\text{imp}}\vert \mbsx^{\text{obs}}) \mathrm{d}\mbsxim \\
    =& \int \nabla_\mbsxim\cdot\{\log p(\mbsx^{\text{imp}}\vert \mbsx^{\text{obs}}) q'^{(k)}(\mbsxim )v_\tau(\mbsxim)\} \mathrm{d}\mbsxim = 0.
    \end{aligned}
\end{equation}
    }
\item{\textbf{Teleportation Step:} In this step, we adjust the weights $\{w_i\}_{i=1}^{N}$ by solving the following optimization problem:
\begin{equation}
\widehat q^{(k+1)} \in \inf_{\widehat q \in 
\widehat{\mathcal{Q}}
}
\left\{
\frac{1}{2}\mathbb{E}_{\widehat q} [\|\nabla\log p(\mbsx^{\text{imp}}\vert \mbsx^{\text{obs}})\|_2^2]
+\frac{1}{2\eta} D_\psi(\widehat q, \tilde{q}^{(k+1)})
\right\},
\end{equation}
When $\psi(\rho) = \int{\rho(\mbsx)[\log\rho(\mbsx) - 1]\dd \mbsx}$, the optimization problem can be reformulated as follows:
\begin{equation}
    \widehat q^{(k+1)} \in \inf_{\widehat q \in 
\widehat{\mathcal{Q}}
}
\left\{
\frac{1}{2}\mathbb{E}_{\widehat q} [\|\nabla\log p(\mbsx^{\text{imp}}\vert \mbsx^{\text{obs}})\|_2^2]
+\frac{1}{2\eta} \int{\widehat{q}(\mbsxim)\log{\frac{\widehat{q}(\mbsxim)}{\tilde{q}^{(k+1)}(\mbsxim)}}\dd \mbsxim } 
\right\}.
\end{equation}
Consider the reaction PDE given by~\Cref{eq:weight_pde}, we have the following result:
\begin{equation}
\begin{cases}
 \widehat{q}(\mbsxim)= \tilde{q}^{(k+1)}(\mbsxim) - \eta \tilde{q}^{(k+1)}(\mbsxim) g_\tau(\mbsxim) + O(\eta) \\
\log{\widehat{q}(\mbsxim)} = \log{\tilde{q}^{(k+1)}(\mbsxim)} - \eta g_\tau(\mbsxim) + \dfrac{\eta^2}{2}\dfrac{\partial^2 \log{\tilde{q}^{(k+1)}(\mbsxim)}}{\partial \tau} + O(\eta^2)
\end{cases}.
\end{equation}
Thus, for the Bregman divergence, we have the following result:
\begin{equation}
\begin{aligned}
  & D_\psi(\widehat q, \tilde{q}^{(k+1)})\\
  = &  \int{\widehat{q}(\mbsxim)\log{\frac{\widehat{q}(\mbsxim)}{\tilde{q}^{(k+1)}(\mbsxim)}}\dd \mbsxim } \\
  = &  \int{\widehat{q}(\mbsxim)[-\eta g_\tau(\mbsxim) + \dfrac{\eta^2}{2}\dfrac{\partial^2 \log{\tilde{q}^{(k+1)}(\mbsxim)}}{\partial \tau}]\dd \mbsxim } + O(\eta^2) \\
  \overset{\text{(i)}}{=} &  - \eta \underbrace{ \int{\tilde{q}^{(k+1)}(\mbsxim) g_\tau(\mbsxim)\dd \mbsxim} }_{=0} + \eta^2 \underbrace{\int{\tilde{q}^{(k+1)}(\mbsxim) g_\tau(\mbsxim)\dd \mbsxim}}_{=0} \\
  &\quad \quad \quad \quad \quad \quad + \dfrac{\eta^2}{2}\int\dfrac{\partial^2 \log{\tilde{q}^{(k+1)}(\mbsxim)}}{\partial \tau}\tilde{q}^{(k+1)}(\mbsxim)\dd \mbsxim  + O(\eta^2) \\
   \overset{\text{(ii)}}{=}  & \int g^2_t(\mbsxim) \tilde{q}^{(k+1)}(\mbsxim) \dd \mbsxim + O(\eta^2)
,
\end{aligned}
\end{equation}
where ``(i)'' is based on the following derivation:
\begin{equation}
\begin{aligned}
   &  \int{\tilde{q}^{(k+1)}(\mbsxim) \dd \mbsxim} = 1 \\
  \Rightarrow  & \dfrac{\dd }{\dd \tau} \int{\tilde{q}^{(k+1)}(\mbsxim) \dd \mbsxim} = 0\\
\Rightarrow   &\int{ \dfrac{\partial \tilde{q}^{(k+1)}(\mbsxim) }{\partial \tau } \dd \mbsxim} =\int [ \int g_\tau(\mbsxim') \tilde{q}^{(k+1)}(\mbsxim') \mathrm{d} \mbsxim' - g_\tau(\mbsxim) ] \tilde{q}^{(k+1)}(\mbsxim) \dd \mbsxim= 0 \\
\Rightarrow   & -\int   g_\tau(\mbsxim) \tilde{q}^{(k+1)}(\mbsxim) \dd \mbsxim= 0 ,
\end{aligned}
\end{equation}
and ``(ii)'' is based on the following derivation: Specifically, we have the following result based on the PDE defined by~\Cref{eq:weight_pde}:
\begin{equation}
\frac{\partial\log \tilde{q}^{(k+1)}(\mbsxim)}{\partial \tau}=\frac{1}{\tilde{q}^{(k+1)}(\mbsxim)}\frac{\partial \tilde{q}^{(k+1)}(\mbsxim)}{\partial \tau}=-g_\tau(\mbsxim)
\Rightarrow
\frac{\partial^2 \log \tilde{q}^{(k+1)}(\mbsxim)}{\partial \tau^2}=-\frac{\partial g_\tau(\mbsxim)}{\partial \tau}.
\end{equation}
Differentiate the identity $\int   g_\tau(\mbsxim) \tilde{q}^{(k+1)}(\mbsxim) \dd \mbsxim= 0 $ in time, we have the following result:
\begin{equation}
\begin{aligned}
 & \frac{\partial}{\partial \tau}\int   g_\tau(\mbsxim) \tilde{q}^{(k+1)}(\mbsxim) \dd \mbsxim = 0\\
\Rightarrow
& \int (\frac{\partial g_\tau(\mbsxim)}{\partial \tau})\tilde{q}^{(k+1)}(\mbsxim)\dd \mbsxim+\int g_\tau(\frac{\partial \tilde{q}^{(k+1)}(\mbsxim)}{\partial \tau}) \dd \mbsxim = 0\\
\Rightarrow
& 
\int (\frac{\partial g_\tau (\mbsxim)}{\partial \tau}) \tilde{q}^{(k+1)}(\mbsxim) \dd \mbsxim-\int g_\tau^2(\mbsxim)\tilde{q}^{(k+1)}(\mbsxim) \dd \mbsxim=0\\
\Rightarrow
& 
\int (\frac{\partial g_\tau (\mbsxim)}{\partial \tau}) \tilde{q}^{(k+1)}(\mbsxim) \dd \mbsxim=\int g_\tau^2(\mbsxim)\tilde{q}^{(k+1)}(\mbsxim) \dd \mbsxim
.
\end{aligned}
\end{equation}
On this basis, we have:
\begin{equation}
\begin{aligned}
   &  \inf_{\widehat q 
}
\frac{1}{2}\mathbb{E}_{\widehat q} [\|\nabla\log p(\mbsx^{\text{imp}}\vert \mbsx^{\text{obs}})\|_2^2]
+\frac{1}{2\eta} \int{\widehat{q}(\mbsxim)\log{\frac{\widehat{q}(\mbsxim)}{\tilde{q}^{(k+1)}(\mbsxim)}}\dd \mbsxim } 
\\
\Rightarrow  &  \inf_{\widehat q
}
\cancel{\frac{1}{2}\mathbb{E}_{\tilde{q}^{(k+1)} } [\|\nabla\log p(\mbsx^{\text{imp}}\vert \mbsx^{\text{obs}})\|_2^2]} - \cancel{\frac{1}{2}\mathbb{E}_{\tilde{q}^{(k+1)} } [\|\nabla\log p(\mbsx^{\text{imp}}\vert \mbsx^{\text{obs}})\|_2^2]}\\
&\quad \quad \quad \quad \quad \quad 
+\frac{1}{\eta}\int{\|\nabla\log p(\mbsx^{\text{imp}}\vert \mbsx^{\text{obs}})\|_2^2 \tilde{q}^{(k+1)}(\mbsxim) \dd \mbsxim}
+\frac{1}{4\eta}\int g_\tau^2(\mbsxim)\tilde{q}^{(k+1)}(\mbsxim) \dd \mbsxim\\
\Rightarrow &
 \inf_{\widehat q
}
\int{\{\|\nabla\log p(\mbsx^{\text{imp}}\vert \mbsx^{\text{obs}})\|_2^2 - \mathbb{E}_{\tilde{q}^{(k+1)}}[\nabla\log p(\mbsx^{\text{imp}}\vert \mbsx^{\text{obs}})]\} \tilde{q}^{(k+1)}(\mbsxim) \dd \mbsxim}
+\frac{1}{4}\int g_\tau^2(\mbsxim)\tilde{q}^{(k+1)}(\mbsxim) \dd \mbsxim \\
\Rightarrow & \inf_{\widehat q
}\mathbb{E}_{\tilde{q}^{(k+1)}}\{ [ g_\tau(\mbsxim) + 2\|\nabla\log p(\mbsx^{\text{imp}}\vert \mbsx^{\text{obs}})\|_2^2 - 2\mathbb{E}_{\tilde{q}^{(k+1)}}[\nabla\log p(\mbsx^{\text{imp}}\vert \mbsx^{\text{obs}})]]^2 \}
.
\end{aligned}
\end{equation}
Thus, the optimal $g_\tau^\star(\mbsxim)$ can be given as follows:
\begin{equation}\label{eq:optimalTeleportDirection}
    g_\tau^\star(\mbsxim) = - 2\|\nabla\log p(\mbsx^{\text{imp}}\vert \mbsx^{\text{obs}})\|_2^2 + 2\mathbb{E}_{\tilde{q}^{(k+1)}}[\| \nabla\log p(\mbsx^{\text{imp}}\vert \mbsx^{\text{obs}})\|_2^2].
\end{equation}
}
\end{itemize}
Based on~\Cref{eq:overallWUpdateResults,eq:overallLogWUpdateResults,eq:overallLogWUpdateResultsUnnormalized,eq:transportVeloctiyFieldImputation,eq:optimalTeleportDirection}, we arrive at the desired results. 
\end{proof}

\subsection{Derivation of Proposition~\ref{prop:softMaxWeightScheme}}\label{subsec:MirrorDescentResults}
\begin{proposition*}[\ref{prop:softMaxWeightScheme}]
Let $\eta>0$ and $\boldsymbol{T}_w\in\mathbb{R}^N \to \mathbb{R}$. Define the intermediate log-weights: $  \log \widehat{{w}}_i^{(k+1)} \coloneqq  \log {{w}}_i^{(k)} +  \eta\boldsymbol{T}(\mbsxim)$, the corresponding normalized weights can be obtained by the following equation:
\begin{equation}
  {{w}}_i^{(k+1)} = \frac{{\widehat{{w}}_i^{(k+1)}}}{\sum_{j=1}^{D}{\widehat{{w}}_j^{(k+1)}}}.
\end{equation}
\end{proposition*}
\begin{proof}
By denoting $\log{\widehat{{w}}} =[\log{\widehat{{w}}}_1, \ldots, \log{\widehat{{w}}}_N]^\top$, we can reformulate $  \log \widehat{{w}}_i^{(k+1)} \coloneqq  \log {{w}}_i^{(k)} +  \eta\boldsymbol{T}_w(\mbsxim)$ as follows:
\begin{equation}\label{eq:bregmanReformulation}
\begin{aligned}
    &  \log \widehat{{w}}_i^{(k+1)} \coloneqq  \log {{w}}_i^{(k)} +  \eta\boldsymbol{T}_w(\mbsxim) \\
  \Rightarrow  & \log \widehat{{w}}^{(k+1)} = \mathbb{\arg\max}_{ \log \widehat{{w}}} \indent 
[\boldsymbol{T}_w(\mbsxim) ]^\top[\log{\widehat{{w}}}]
-\frac{1}{2\eta}\|\log{\widehat{{w}}_i}-\log{{w}_i}\|_2^2.
%
\end{aligned}
\end{equation}
We observe that the major issue that results in the ill-defined iteration result is the introduction of the Euclidean distance as the proximal term. Thus, the key for addressing this issue is replacing the proximal term with entropy-induced Bregman divergence. On this basis, we introduce the Bregman divergence for $\widehat{w}$ and $w$ as follows:
\begin{equation}
\begin{aligned}
    & D_\psi(\widehat{w}, w^{(k)}) = \psi( \widehat{w} ) -\psi( w^{(k)}) - [\nabla\psi(w^{(k)})]^\top[ \widehat{w} - w^{(k)} ] \\
\overset{\psi(w)\coloneqq w^\top[\log{w} - 1]}{\Longrightarrow}   & D_\psi(\widehat{w}, w^{(k)}) =\sum_{i=1}^{D} \widehat{w}_i \log\frac{\widehat{w}_i}{w_i^{(k)}}-\sum_{i=1}^{D} \widehat{w}_i+\sum_{i=1}^{D} w_i^{(k)}.
\end{aligned}
\end{equation}

Now the optimization problem can be given as follows:
\begin{equation}
    \begin{aligned}
    &   \mathop{\arg\max}_{  {{w}}} \indent 
[\boldsymbol{T}_w(\mbsxim) ]^\top[w]
-\frac{1}{\eta}\{\sum_{i=1}^{D} w_i\log\frac{w_i}{w_i^{(k)}}-\sum_{i=1}^{D} w_i +\sum_{i=1}^{D} w_i^{(k)}\} \\
   \mathrm{s.t.} & \begin{cases}\sum_{j=1}^{D}{{w}}_j = 1 \\{{w}}_j \ge 0, \forall j\in\{1,\ldots,D\}
   \end{cases}
    \end{aligned}.
\end{equation}

Introducing the Lagrange multiplier $\lambda\in\mathbb R$ and $\gamma\in\mathbb{R}^{D}$ for the constraint $\sum_{j=1}^{D}\widehat{{w}}_j = 1$ and $\widehat{{w}}_j \ge 0,  \forall j\in\{1,\ldots,D\}$. The Lagrangian can be given as follows:
\begin{equation}
\mathcal L(w,\lambda,\gamma)
=
[\boldsymbol{T}_w(\mbsxim)]^\top w
-\frac{1}{\eta}\{\sum_{i=1}^{D} w_i\log\frac{w_i}{w_i^{(k)}}-\sum_{i=1}^{D} w_i +\sum_{i=1}^{D} w_i^{(k)}\}
+\lambda(\sum_{i=1}^{D} w_i-1)
+\sum_{i=1}^{D}\gamma_i w_i.
\end{equation}

For each $i\in\{1,\ldots,D\}$, using
\begin{equation}
\frac{\partial}{\partial w_i}(w_i\log\frac{w_i}{w_i^{(k)}})=\log\frac{w_i}{w_i^{(k)}}+1,
\qquad
\frac{\partial}{\partial w_i}(-\sum_j w_j)=-1,
\end{equation}
we get the following result:
\begin{equation}\label{eq:firstOrderResult}
\frac{\partial \mathcal L}{\partial w_i}
=
\boldsymbol{T}_{w_i}(\mbsxim)
-\frac{1}{\eta}(\log\frac{w_i}{w_i^{(k)}}+1-1)
+\lambda+\gamma_i
=
\boldsymbol{T}_{w_i}(\mbsxim)-\frac{1}{\eta}\log\frac{w_i}{w_i^{(k)}}+\lambda+\gamma_i.
\end{equation}
Setting~\Cref{eq:firstOrderResult} to $0$, we have the following result:
\begin{equation}\label{eq:firstOrderConditionBregman}
0=\boldsymbol{T}_{w_i}(\mbsxim)-\frac{1}{\eta}\log\frac{w_i}{w_i^{(k)}}+\lambda+\gamma_i
\quad\Longrightarrow\quad
\log\frac{w_i}{w_i^{(k)}}=\eta(\boldsymbol{T}_{w_i}(\mbsxim)+\lambda+\gamma_i).
\end{equation}

For the Karush–Kuhn–Tucker condition~\citep{boyd2004convex}, we have the following conditions termed ``complementary slackness'' condition:
\begin{equation}
w_i\ge 0,\quad \gamma_i\ge 0,\quad \gamma_i w_i=0,\quad \sum_i w_i=1.
\end{equation}
Since $w_i$ can be given as follows based on~\Cref{eq:firstOrderConditionBregman}:
\begin{equation}
    w_i = w_i^{(k)}\exp[\eta(\boldsymbol{T}_{w_i}(\mbsxim)+\lambda+\gamma_i)],
\end{equation}
we know that $\gamma_i=0$ for all $i\in\{1,\ldots, D\}$. Then, we have:
\begin{equation}\label{eq:wMovingResult}
 w_i = w_i^{(k)}\exp[\eta(\boldsymbol{T}_{w_i}(\mbsxim)+\lambda)].
\end{equation}
On this basis, using the fact that:
\begin{equation}\label{eq:widehatWDefinition}
    \widehat w_i^{(k+1)}\coloneqq w_i^{(k)}\exp(\eta \boldsymbol{T}_{w_i}(\mbsxim)),
\end{equation}
we can further reformulate \Cref{eq:wMovingResult} as follows:
\begin{equation}\label{eq:wWithCoefficientResult}
w_i^{(k+1)}=\exp(\eta\lambda)\widehat w_i^{(k+1)}.
\end{equation}
Imposing $\sum_{i=1}^D w_i^{(k+1)}=1$, the coefficient $\exp(\eta\lambda)$ can be reformulated as follows:
\begin{equation}\label{eq:lambdaExpressionResult}
\exp(\eta\lambda)=\frac{1}{\sum_{j=1}^{D}\widehat w_j^{(k+1)}}.
\end{equation}
Finally, based on~\Cref{eq:widehatWDefinition,eq:lambdaExpressionResult,eq:wWithCoefficientResult}, we arrive at the desired result.
\end{proof}

\subsection{Derivation of Proposition~\ref{prop:scoreNetworkLearningObjective} and Related Discussions}\label{subsec:derivationDSMProp}
\begin{proposition*}
The following two learning objective are identical for score network $s_\theta(\mbsxim)$ learning:
\begin{equation}
\begin{aligned}
    \mathop{\arg\min}_{ s_\theta} \| s_\theta(\mbsxim) - \nabla\log{p(\mbsxim\vert\mbsxobs)}\|_2^2 
   =  \mathop{\arg\min}_{ s_\theta} \| s_\theta(\mbsxim) - \nabla\log{p(\mbsxim)}\|_2^2
\end{aligned}
\end{equation}
\end{proposition*}

\begin{proof}
The key is proving that learning $\nabla\log{p(\mbsxim)}$ is identity to learning $\nabla\log{p(\mbsxim|\mbsxobs)}$. To address this issue, we have the following result:
\begin{equation}
 p(\mbsxim)  = \int{p(\mbsxobs)p(\mbsxim|\mbsxobs)}  =  \mathbb{E}_{p(\mbsxobs)}[\log{p(\mbsxim|\mbsxobs)}]. 
\end{equation}
On this basis, the score function of $p(\mbsxim)$ can be written as
\begin{equation}
\begin{aligned}
\nabla\log p(\mbsxim)\overset{\text{(i)}}{=}
\frac{\mathbb{E}_{p(\mbsxobs)}[\log{p(\mbsxim|\mbsxobs)}]}{p(\mbsxim)} =
\frac{\mathbb{E}_{p(\mbsxobs)}[p(\mbsxim|\mbsxobs)\nabla\log{p(\mbsxim|\mbsxobs)}]}{p(\mbsxim)}
\overset{\text{(ii)}}{=}
\mathbb E_{p(\mbsxobs | \mbsxim)}[\nabla\log{p(\mbsxim|\mbsxobs)}],
\end{aligned}
\end{equation}
where ``(i)'' follows from differentiating under the integral sign:
\begin{equation}
\begin{aligned}
\nabla\log p(\mbsxim)
=\frac{\nabla p(\mbsxim)}{p(\mbsxim)}
= \frac{\nabla\int p(\mbsxobs)p(\mbsxim|\mbsxobs) \mathrm d \mbsxobs}{p(\mbsxim)} = \frac{\int p(\mbsxobs)\nabla p(\mbsxim|\mbsxobs)\mathrm d \mbsxobs}{p(\mbsxim)}
= \frac{\mathbb E_{p(\mbsxobs)}\left[\nabla p(\mbsxim|\mbsxobs)\right]}{p(\mbsxim)} ,
\end{aligned}
\end{equation}
and ``(ii)'' uses Bayes' rule $p(\mbsxobs| \mbsxim)=\dfrac{p(\mbsxim \vert  \mbsxobs)p(\mbsxobs)}{p(\mbsxim)}$:
\begin{equation}
\begin{aligned}
&\frac{\mathbb E_{p(\mbsxobs)}\left[p(\mbsxim| \mbsxobs)\nabla_{\mbsxim}\log p(\mbsxim| \mbsxobs)\right]}{p(\mbsxim)}\\
=&\frac{\int p(\mbsxim| \mbsxobs)p(\mbsxobs)\nabla_{\mbsxim}\log p(\mbsxim| \mbsxobs)\mathrm d \mbsxobs}{p(\mbsxim)}\\
=&\int p(\mbsxobs| \mbsxim)\nabla_{\mbsxim}\log p(\mbsxim| \mbsxobs)\mathrm d \mbsxobs.
\end{aligned}
\end{equation}
Moreover, by the law of tower property, as given by Theorem 34.4 of reference~\citep{billingsley2012probability}, for any integrable function $h(\mbsxobs,\mbsxim)$,
\begin{equation}
\mathbb E_{p(\mbsxim)}\left[\mathbb E_{p(\mbsxobs| \mbsxim)}[h(\mbsxobs,\mbsxim)]\right]
=\mathbb E_{p(\mbsxobs,\mbsxim)}[h(\mbsxobs,\mbsxim)].
\end{equation}
Taking $g(\mbsxobs,\mbsxim)=\nabla_{\mbsxim}\log p(\mbsxim| \mbsxobs)$ yields
\begin{equation}
\mathbb E_{p(\mbsxim)}\left[\mathbb E_{p(\mbsxobs| \mbsxim)}[\nabla_{\mbsxim}\log p(\mbsxim| \mbsxobs)]\right]
=\mathbb E_{p(\mbsxobs,\mbsxim)}[\nabla_{\mbsxim}\log p(\mbsxim| \mbsxobs)].
\end{equation}
Furthermore, in our TSDI setting given in~\Cref{subsec:problemFormulationResult}, sampling $\mbsxobs$ can obtain the sample from $\mbsxim$. Based on this, we have the following result:
\begin{equation}
   \mathbb E_{p(\mbsxim)}[\nabla\log p(\mbsxim) ] = \mathbb E_{p(\mbsxim)}[\nabla\log p(\mbsxim| \mbsxobs)].
\end{equation}

As such, we introduce the score network $s_\theta(\mbsxim)$ with parameter $\theta$ and formulate the following learning objective:
\begin{equation}\label{eq:score_object_transfer1}
    \mathop{\arg\min}_{s_\theta} s^\top_\theta(\mbsxim)\nabla\log p(\mbsxim| \mbsxobs) + \frac{1}{2}\Vert s_\theta(\mbsxim) \Vert_2^2=   \mathop{\arg\min}_{s_\theta} s^\top_\theta(\mbsxim)\nabla\log p(\mbsxim)+ \frac{1}{2}\Vert s_\theta(\mbsxim) \Vert_2^2, 
\end{equation}
where the $L_2$-norm is added to realize the regularity condition for score function~\citep{vincent2011connection}. Notably, the left-hand-side of~\Cref{eq:score_object_transfer1} can be further reformulated as follows:
\begin{equation}\label{eq:score_object_transfer2}
\begin{aligned}
    &  \mathop{\arg\min}_{s_\theta}  \quad s^\top_\theta(\mbsxim)\nabla\log p(\mbsxim| \mbsxobs) + \frac{1}{2}\Vert s_\theta(\mbsxim) \Vert_2^2 \\
      =   &  \mathop{\arg\min}_{s_\theta}  \quad s^\top_\theta(\mbsxim)\nabla\log p(\mbsxim| \mbsxobs) + \frac{1}{2}\Vert s_\theta(\mbsxim) \Vert_2^2 + \underbrace{\frac{1}{2}\Vert \nabla\log p(\mbsxim| \mbsxobs) \Vert_2^2}_{\text{constant}}\\
      =   &  \mathop{\arg\min}_{s_\theta}  \quad \frac{1}{2}\Vert s(\mbsxim) - \nabla\log p(\mbsxim| \mbsxobs) \Vert_2^2
      ,
\end{aligned}
\end{equation}
and the right-hand-side of~\Cref{eq:score_object_transfer1} can be further reformulated as follows:
\begin{equation}\label{eq:score_object_transfer3}
\begin{aligned}    & 
      \mathop{\arg\min}_{s_\theta} \quad s^\top_\theta(\mbsxim)\nabla\log p(\mbsxim)+ \frac{1}{2}\Vert s_\theta(\mbsxim) \Vert_2^2 \\ 
= &   \mathop{\arg\min}_{s_\theta}  \quad s^\top_\theta(\mbsxim)\nabla\log p(\mbsxim)+ \frac{1}{2}\Vert s_\theta(\mbsxim) \Vert_2^2+ \underbrace{\frac{1}{2}\Vert \nabla\log p(\mbsxim) \Vert_2^2}_{\text{constant}} \\
= &   \mathop{\arg\min}_{s_\theta} \quad  \frac{1}{2}\Vert s_\theta(\mbsxim) - \nabla\log p(\mbsxim)  \Vert_2^2.
    \end{aligned}
\end{equation}
Based on~\Cref{eq:score_object_transfer2,eq:score_object_transfer3}, we arrive at the desired result.
\end{proof}

Notably, in the main text we show that learning the right-hand side of \Cref{eq:dsmMainContent} is equivalent to optimizing the objective in \Cref{eq:dsmTargetFunctionAppendix}, following the denoising score matching formulation of \citet{vincent2011connection}: 
\begin{equation}\label{eq:dsmMainContent}
\begin{aligned}
  \mathop{\arg\min}_{ s_\theta} \| s_\theta(\mbsxim) - \nabla\log{p(\mbsxim\vert\mbsxobs)}\|_2^2  = \mathop{\arg\min}_{ s_\theta} \| s_\theta(\mbsxim) - \nabla\log{p(\mbsxim)}\|_2^2.
\end{aligned}
\end{equation}
\begin{equation}\label{eq:dsmTargetFunctionAppendix}
\mathop{\arg\min}_{s_\theta}\mathbb{E}_{q_{\sigma}(\widehat{\mbsx}^{\text{imp}}\vert\mbsxim)}[\Vert s_\theta(\widehat{\mbsx}^{\text{imp}}) - \nabla\log q_{\sigma}(\widehat{\mbsx}^{\text{imp}} \vert\mbsxim)\Vert_2^2 ],
\end{equation}

To maintain the rigor of this manuscript, we provide the corresponding derivation, following \citet{vincent2011connection}, to justify this equivalence. For the objective in \Cref{eq:dsmMainContent}, we obtain the following derivation:

\begin{equation}
 \begin{aligned}
&\mathbb{E}_{p(\mathbf{x}^{\text{imp}})} [ \frac{1}{2} \left\| s(\mathbf{x}^{\text{imp}}) - \nabla_{\mathbf{x}^{\text{imp}}}\log p(\mathbf{x}^{\text{imp}}) \right\|^2 ]\\
=& \mathbb{E}_{p(\mathbf{x}^{\text{imp}})} [ \int q_\sigma(\widehat{\mathbf{x}}^{\text{imp}}|\mathbf{x}^{\text{imp}}) \frac{1}{2} \Vert s(\mathbf{x}^{\text{imp}}) - \nabla_{\mathbf{x}^{\text{imp}}}\log p(\mathbf{x}^{\text{imp}}) \Vert^2_2 d\widehat{\mathbf{x}}^{\text{imp}}  ]  \\
= & \mathbb{E}_{p(\mathbf{x}^{\text{imp}})} \mathbb{E}_{q_\sigma(\widehat{\mathbf{x}}^{\text{imp}}|\mathbf{x}^{\text{imp}})} [ \frac{1}{2} \Vert s(\mathbf{x}^{\text{imp}}) - \nabla_{\mathbf{x}^{\text{imp}}}\log p(\mathbf{x}^{\text{imp}}) \Vert^2_2  ] \\
\overset{\text{(i)}}{\approx}& \mathbb{E}_{p(\mathbf{x}^{\text{imp}})} \mathbb{E}_{q_\sigma(\widehat{\mathbf{x}}^{\text{imp}}|\mathbf{x}^{\text{imp}})} [ \frac{1}{2}  \Vert s(\widehat{\mathbf{x}}^{\text{imp}}) - \nabla_{\widehat{\mathbf{x}}^{\text{imp}}}\log q_\sigma(\widehat{\mathbf{x}}^{\text{imp}})  \Vert^2_2  ] \\
 \overset{\text{(ii)}}{=} &\mathbb{E}_{p(\mathbf{x}^{\text{imp}})}\mathbb{E}_{q_\sigma(\widehat{\mathbf{x}}^{\text{imp}}|\mathbf{x}^{\text{imp}})} \{ \frac{1}{2} \Vert s(\widehat{\mathbf{x}}^{\text{imp}}) - \mathbb{E}{q(\mathbf{x}^{\text{imp}}|\widehat{\mathbf{x}}^{\text{imp}})}[\nabla{\widehat{\mathbf{x}}^{\text{imp}}}\log q_\sigma(\widehat{\mathbf{x}}^{\text{imp}}|\mathbf{x}^{\text{imp}})] \Vert^2_2 \}\\
\overset{\text{(iii)}}{\approx}  &\mathbb{E}_{p(\mathbf{x}^{\text{imp}})}\mathbb{E}_{q_\sigma(\widehat{\mathbf{x}}^{\text{imp}}|\mathbf{x}^{\text{imp}})} [ \frac{1}{2} \Vert s(\widehat{\mathbf{x}}^{\text{imp}}) - \nabla_{\widehat{\mathbf{x}}^{\text{imp}}}\log q_\sigma(\widehat{\mathbf{x}}^{\text{imp}}|\mathbf{x}^{\text{imp}}) \Vert ^2_2  ] 
 ,
\end{aligned}
\end{equation}
where ``(i)'' is based on the following equation:
\begin{equation}
    q_\sigma(\widehat{\mbsx}^{\text{imp}}) =\int{p(\mbsxim)q_\sigma(\widehat{\mbsx}^{\text{imp}}\vert \mbsxim)\mathrm{d}\mbsxim} =\int{p(\mbsxim)\mathcal{N}(\widehat{\mbsx}^{\text{imp}},\sigma^2I)\mathrm{d}\mathbf{x}} ,
\end{equation}
``(ii)'' is based on the following equation:
\begin{equation}
    \begin{aligned}
\nabla \log q_\sigma(\widehat{\mbsx}^{\text{imp}})
= &\dfrac{\nabla q_\sigma(\widehat{\mbsx}^{\text{imp}})}{q_\sigma(\widehat{\mbsx}^{\text{imp}})} \\
=& \frac{\nabla \int p(\mbsxim) q_\sigma(\widehat{\mathbf{x}}^{\text{imp}}|\mathbf{x}^{\text{imp}}) \mathrm{d}\mathbf{x}^{\text{imp}}}{q_\sigma(\widehat{\mathbf{x}}^{\text{imp}})} \\
=& \dfrac{ \int p(\mathbf{x}^{\text{imp}}) \nabla q_\sigma(\widehat{\mathbf{x}}^{\text{imp}}|\mathbf{x}^{\text{imp}}) \mathrm{d}\mathbf{x}^{\text{imp}} }{ q_\sigma(\widehat{\mathbf{x}}^{\text{imp}}) } \\
= &\frac{ \int p(\mathbf{x}^{\text{imp}}) q_\sigma(\widehat{\mathbf{x}}^{\text{imp}}|\mathbf{x}^{\text{imp}}) \nabla \log q_\sigma(\widehat{\mathbf{x}}^{\text{imp}}|\mathbf{x}^{\text{imp}}) \mathrm{d}\mathbf{x}^{\text{imp}} }{ q_\sigma(\widehat{\mathbf{x}}^{\text{imp}}) } \\
=& \int \frac{ p(\mathbf{x}^{\text{imp}}) q_\sigma(\widehat{\mathbf{x}}^{\text{imp}}|\mathbf{x}^{\text{imp}}) }{ q_\sigma(\widehat{\mathbf{x}}^{\text{imp}}) } \nabla \log q_\sigma(\widehat{\mathbf{x}}^{\text{imp}}|\mathbf{x}^{\text{imp}}) \mathrm{d}\mathbf{x}^{\text{imp}} \\
=& \int q(\widehat{\mathbf{x}}^{\text{imp}}|\mathbf{x}^{\text{imp}}) \nabla \log q_\sigma(\widehat{\mathbf{x}}^{\text{imp}}|\mathbf{x}^{\text{imp}}) \mathrm{d}\mathbf{x}^{\text{imp}} \\
=& \mathbb{E}_{q(\widehat{\mathbf{x}}^{\text{imp}}|\mathbf{x}^{\text{imp}})} \left[ \nabla \log q_\sigma(\widehat{\mathbf{x}}^{\text{imp}}|\mathbf{x}^{\text{imp}}) \right],
\end{aligned}
\end{equation}
and (iii) is based on the fact that: The expectation $\mathbb{E}_{q(\widehat{\mathbf{x}}^\text{(imp)}|\mathbf{x}^\text{(imp)})} [\nabla\log q(\widehat{\mathbf{x}}^\text{(imp)}|\mathbf{x}^\text{(imp)}) ]$ with respect to $q(\widehat{\mathbf{x}}^\text{(imp)}|\mathbf{x}^\text{(imp)})$, for a given $\widehat{\mathbf{x}}^{\text{imp}}$, is the average of the conditional scores over all possible data points $\mathbf{x}^{\text{imp}}$ that could have produced this noisy observation $\widehat{\mathbf{x}}^{\text{imp}}$.

\subsection{Discussions of the Convergence Property for SPIRIT Framework}\label{subsec:discussionsOnTheConvergence}


In \Cref{subsec:overallWorkFlowSPIRIT}, we note that the ``Score Learning'' stage admits convergence guarantees from standard optimization theory under mild regularity conditions (e.g., \citep{bottou2018optimization}). For the ``Recursive Imputation'' stage, define the energy $\mathcal{J}(\mbsx^{\text{imp}})\coloneqq\mathbb{E}_{q'}[\log p(\mbsx^{\text{imp}}\mid \mbsx^{\text{obs}})] + \mathbb{E}_{q'}[\left\|\nabla \log p(\mbsx^{\text{imp}}\mid \mbsx^{\text{obs}})\right\|_2^2].$ Assuming $\mathcal{J}$ is lower bounded and smooth, its value decreases monotonically along the imputation iterates in the continuous-time limit $\eta\to 0$, and the dynamics converge to a stationary point. In this subsection, we further analyze the convergence properties of both the ``Score Learning'' and ``Recursive Imputation'' stages.

\paragraph{Convergence Analysis of ``Score Learning''.} At first, we define the $\mathcal{L}^{\text{DSM}}$ as follows:
\begin{equation}
   \mathcal{L}^{\text{DSM}}(\theta )\coloneqq  \mathbb{E}_{q_{\sigma}(\widehat{\mbsx}^{\text{imp}}\vert\mbsxim)}[\Vert s_\theta(\mbsxim) - \nabla\log q_{\sigma}(\widehat{\mbsx}^{\text{imp}} \vert\mbsxim)\Vert_2^2 ]
\end{equation}
For the following mild assumptions, we get the convergence promising of the ``Score Learning'' stage:
\begin{enumerate}

\item{\textbf{Assumption 1 (Lower bound):} $\mathcal L^{\text{DSM}}(\theta)\ge \mathcal L^{\text{DSM}}(\theta^\star) \ge 0 > -\infty$ for all $\theta$. }

\item{\textbf{Assumption 2 ($L$-smooth):} $\nabla \mathcal L^{\text{DSM}}(\theta)$ is $L$-Lipschitz: $\|\nabla  \mathcal L^{\text{DSM}}(\theta)-\nabla  \mathcal L^{\text{DSM}}(\theta')\|\le L\|\theta-\theta'\|\quad\forall \theta,\theta'.$
}
\end{enumerate}

The ``Lower bound'' condition holds since $\mathcal L^{\mathrm{DSM}}(\theta)$ is an expectation of a squared $\ell_2$ error and thus $\mathcal L^{\mathrm{DSM}}(\theta)\ge 0$. The ``$L$-smoothness'' condition (i.e., Lipschitz continuity of $\nabla \mathcal L^{\mathrm{DSM}}$) is a standard assumption in gradient-descent convergence analyses, for example \citep{bach2024learning}. A sufficient (but not necessary) set of conditions for $L$-smoothness is to impose appropriate boundedness/spectral-norm constraints on the network so that $\nabla \mathcal L^{\mathrm{DSM}}$ becomes Lipschitz; gradient clipping is a practical heuristic that controls update magnitudes but does not imply $L$-smoothness.


Under Assumptions 1 and 2, for learning rate $lr \in(0,\frac{2}{L})$, the gradient descent iterates satisfy:
\begin{equation}
    \|\nabla   \mathcal{L}^{\text{DSM}}(\theta_e)\|\to 0 \quad \text{as } e\to\infty.
\end{equation}

Because $ \mathcal{L}^{\text{DSM}}(\theta)$ is $L$-smooth, we have the following result for any $\theta,\theta'$,
\begin{equation}
 \mathcal{L}^{\text{DSM}}(\theta')\le  \mathcal{L}^{\text{DSM}}(\theta) + \langle \nabla  \mathcal{L}^{\text{DSM}}(\theta),\theta'-\theta\rangle + \frac{L}{2}\|\theta'-\theta\|^2.
\end{equation}
Apply this with $\theta=\theta_e$ and $\theta'=\theta_{e+1}=\theta_e-lr \nabla  \mathcal{L}^{\text{DSM}}(\theta_e)$. Then
\begin{equation}
\begin{aligned}
&  \mathcal{L}^{\text{DSM}}(\theta_{e+1})\\
\le & \mathcal{L}^{\text{DSM}}(\theta_e)
+\left\langle  \mathcal{L}^{\text{DSM}}(\theta_e), -lr \nabla  \mathcal{L}^{\text{DSM}}(\theta_e)\right\rangle
+\frac{L}{2}\| -\xi \nabla  \mathcal{L}^{\text{DSM}}(\theta_e)\|^2\\
= & \mathcal{L}^{\text{DSM}}(\theta_e) - lr \|\nabla  \mathcal{L}^{\text{DSM}}(\theta_e)\|^2
+\frac{L \times lr^2}{2}\| \mathcal{L}^{\text{DSM}}(\theta_e)\|^2\\
= & \mathcal{L}^{\text{DSM}}(\theta_e) - lr(1-\frac{L\times lr }{2})\|\nabla  \mathcal{L}^{\text{DSM}}(\theta_e)\|^2.
\end{aligned}
\end{equation}
Since $lr\in(0,\frac{2}{L})$, we define the coefficient $c$ as follows:
\begin{equation}
c\coloneqq lr (1-\frac{L\times lr}{2} )>0.
\end{equation}
Therefore, we get the following results:
\begin{equation}
\mathcal{L}^{\text{DSM}}(\theta_e)-\mathcal{L}^{\text{DSM}}(\theta_{e+1}) \ge c \|\nabla \mathcal{L}^{\text{DSM}}(\theta_e)\|^2.
\end{equation}
Summing from $t=0$ to ${\mathcal{E}_{\text{score}}}-1$ yields a telescoping sum:
\begin{equation}
\sum_{e=0}^{{\mathcal{E}_{\text{score}}}-1} c\|\nabla \mathcal{L}^{\text{DSM}}(\theta_e)\|^2
\le  \mathcal{L}^{\text{DSM}}(\theta_e)-\mathcal{L}^{\text{DSM}}(\theta_\mathrm{T})
\le \mathcal{L}^{\text{DSM}}(\theta_e)-\mathcal{L}^{\text{DSM}}(\theta^\star).
\end{equation}
Letting ${\mathcal{E}_{\text{score}}}\to\infty$,
\begin{equation}
\sum_{t=0}^{\infty} \|\nabla \mathcal{L}^{\text{DSM}}(\theta_e)\|^2 < \infty.
\end{equation}
A series of nonnegative terms is finite only if the terms go to zero, hence
\begin{equation}
\lim_{e\to\infty}\|\nabla \mathcal{L}^{\text{DSM}}(\theta_e)\| \to 0 .
\end{equation}
This proves convergence to a stationary point.

\paragraph{Convergence Analysis of ``Recursive Imputation''.}
When $\eta\to 0$, based on~\Cref{sec:additionalBackgroundKnowledge}, we are implicitly simulating the following PDE using the Forward-Euler method:
\begin{equation}\label{eq:reactionDiffusionPDE}
   \frac{\partial q_\tau(\mbsxim)}{\partial \tau} = - \nabla\cdot (v_\tau^\star(\mbsxim)q_\tau(\mbsxim) ) + g_\tau^\star(\mbsxim).
\end{equation}
On this basis, the change of functional $\mathcal{J}(q')$ can be given as follows:
\begin{equation}\label{eq:functionalAlongTimeT}
\begin{aligned}
   &  \frac{\dd}{\dd \tau} \mathcal{J}(q'_\tau) \\
   = & -\int{\nabla[\cdot(q'_
   \tau(\mbsxim) v_\tau^\star(\mbsxim))\delta_{q'_\tau}\mathcal{J}(q'_\tau)]\dd\mbsxim} 
   + \int{q'_\tau(\mbsxim)g_\tau(\mbsxim) \delta_{q'_\tau}\mathcal{J}(q'_\tau) \dd\mbxim} \\
   \overset{\text{(i)}}{=} & \int{ [v_\tau^\star(\mbsxim)]^\top[\nabla\delta_{q'_\tau}\mathcal{J}(q'_\tau)]q'_
   \tau(\mbsxim)\dd\mbsxim} 
   + \int{q'_\tau(\mbsxim)g_\tau(\mbsxim) \delta_{q'_\tau}\mathcal{J}(q'_\tau) \dd\mbxim} ,
\end{aligned}
\end{equation}
where ``(i)'' is based on the integration-by-parts~\citep{8744312,dong2022particle}. According to the the theoretical derivation, for $\eta\to 0$, we have:
\begin{subequations}
\begin{align}
    v_\tau^\star(\mbsxim) =&-\nabla\delta_{q'_\tau}\mathcal{J}(q'_\tau), \label{eq:wassVelocityField} \\
g_\tau^\star(\mbsxim) =& -\delta_{q'_\tau}\mathcal{J}(q'_\tau) + \mathbb{E}_{q'_\tau}[\delta_{q'_\tau}\mathcal{J}(q'_\tau)] \label{eq:fisherVelocityField}.
\end{align}
\end{subequations}
Plugging~\Cref{eq:wassVelocityField,eq:fisherVelocityField} into~\Cref{eq:functionalAlongTimeT}, we have the following result~\citep{neklyudov2023wasserstein}:
\begin{equation}
    \begin{aligned}
    \frac{\dd}{\dd \tau} \mathcal{J}(q'_\tau) =& -\int{\Vert \nabla\delta_{q'_\tau}\mathcal{J}(q'_\tau)\Vert_2^2\dd\mbsxim} 
   - \int{\{\delta_{q'_\tau}\mathcal{J}(q'_\tau)-\mathbb{E}_{q'_\tau}[\delta_{q'_\tau}\mathcal{J}(q'_\tau)] \}^2q'_\tau(\mbsxim)\dd\mbxim} \\
   &\indent - \underbrace{ \int{\{\delta_{q'_\tau}\mathcal{J}(q'_\tau) - \mathbb{E}_{q'_\tau}[\delta_{q'_\tau}\mathcal{J}(q'_\tau)] \}\{\mathbb{E}_{q'_\tau}[\delta_{q'_\tau}\mathcal{J}(q'_\tau)]\}q'_\tau(\mbsxim)\dd\mbsxim} }_{=0} \\
   \le & 0,
\end{aligned}
\end{equation}
Integration the Reaction PDE given by~\Cref{eq:weight_pde} on both sides, we have:
\begin{equation}
    \int{\frac{\partial q'_\tau(\mbsxim)}{\partial \tau}\dd\mbsxim} = - \int{\{\delta_{q'_\tau}\mathcal{J}(q'_\tau) - \mathbb{E}_{q'_\tau}[\delta_{q'_\tau}\mathcal{J}(q'_\tau)] \}q'_\tau(\mbsxim)\dd\mbsx} = 0,
\end{equation}

Thus, the functional $\mathcal J[q']$ decreases along the PDE flow: when the density $q'_\tau$ evolves according to \Cref{eq:reactionDiffusionPDE}, $\mathcal J[q'_\tau]$ is non-increasing and the dynamics converge to a stationary solution characterized by the following derivation:
\begin{equation}
\bigl\|\nabla_x \delta_{q'_\tau}\mathcal J[q'_\tau]\bigr\|_2^2 = 0
\Rightarrow
\nabla_x \delta_{q'_\tau}\mathcal J[q'_\tau]=0
\Rightarrow
\delta_{q'_\tau}\mathcal J[q'_\tau]\equiv C,
\end{equation}
where the last implication follows because a function with zero spatial gradient is a constant, and $C<\infty$ is a constant. Therefore, we establish convergence of the “Recursive Imputation” stage in the continuous-time limit.


\section{Detailed Experimental Protocols}

\subsection{Dataset Information}\label{subsec:datasetDescription}
Our empirical evaluation is conducted on a diverse collection of widely-used time series benchmarks. Each dataset presents distinct characteristics in terms of dimensionality and temporal resolution. A summary is provided in \Cref{tab:dataset}. The detailed information of these datasets are listed as follows:
\begin{itemize}[leftmargin=*]
    \item{\textbf{ETT}~\citep{zhou2021informer}: Contains seven metrics related to electricity transformers, recorded from July 2016 to July 2018. It is divided into four subsets based on sampling frequency: ETTh1 and ETTh2 (hourly), and ETTm1 and ETTm2 (every 15 minutes).}

    \item{\textbf{Exchange}~\citep{xu2021autoformer}: This dataset collects the daily exchange rates of 8 countries, including Australia, British, Canada, Switzerland, China, Japan, New Zealand and Singapore, from 1990.01 to 2016.12.}
    \item{\textbf{Illness}~\citep{xu2021autoformer}: This dataset describes the ratio of patients seen with influenza-like illness and the total number of the patients. It includes the weekly data from the Centers for Disease Control and Prevention of the United States from 2002 to 2021.}
    \item \textbf{Traffic}~\citep{xu2021autoformer}: Documents the hourly occupancy rates of 862 sensors on San Francisco Bay Area freeways, spanning from 2015 to 2016.
\end{itemize}

\begin{table}[!h]
\caption{Dataset information. }\label{tab:dataset}
 \centering
 \small\setlength{\tabcolsep}{3.2pt}\renewcommand\arraystretch{1.3}
\begin{threeparttable}
\begin{tabular}{llllll}
    \toprule
    Dataset & $D$ & Patch Length & Dataset Size & Frequency& Domain \\
    \toprule
     ETT-h1 & 7 & 24 & 17420 & Hourly & Temperature for transformer oil \\
     \midrule
     ETT-h2 & 7 &24 & 17420 & Hourly & Temperature for transformer oil\\
     \midrule
     ETT-m1 & 7 & 24 & 69680 & 15-minute & Temperature for transformer oil\\
     \midrule
     ETT-m2 & 7 & 24 & 69680 & 15-minute & Temperature for transformer oil\\
    \midrule
       Exchange & 8 & 24 & 7558 & Daily & Financial \\
    \midrule
         Illness & 7 & 24 & 966 & Weekly & Health\\
    \midrule
    Traffic & 862 & 24 & 236 & Hourly & Transportation \\
    \bottomrule
\end{tabular}
    \begin{tablenotes}
    \item  \scriptsize \textit{Kindly Note}:  \textit{D} denotes the number of variates. \emph{Frequency} denotes the sampling interval of time points.
    \end{tablenotes}
\end{threeparttable}
\end{table}

\subsection{Model Hyperparameters}\label{subsec:hyperParamsSettings}

In this manuscript, we adopt a set of widely used TSDI methods as baselines to evaluate the effectiveness of the proposed SPIRIT framework, including Crossformer~\citep{zhang2023crossformer}, TimesNet~\citep{Timesnet}, PatchTST~\citep{niePatchTST}, Autoformer~\citep{xu2021autoformer}, ETSformer~\citep{woo2022etsformer}, FiLM~\citep{zhou2022film}, DLinear~\citep{zeng2023transformers}, GP-VAE~\citep{fortuin2020gp}, CSDI~\citep{tashiro2021csdi}, Glocal~\citep{glocalImputation}, Sinkhorn~\citep{muzellec2020missing}, TDM~\citep{zhao2023transformed}, and PSW-I~\citep{wang2025optimal}. 
For Crossformer~\citep{zhang2023crossformer}, TimesNet~\citep{Timesnet}, PatchTST~\citep{niePatchTST}, Autoformer~\citep{xu2021autoformer}, ETSformer~\citep{woo2022etsformer}, FiLM~\citep{zhou2022film}, DLinear~\citep{zeng2023transformers}, GP-VAE~\citep{fortuin2020gp}, and CSDI~\citep{tashiro2021csdi}, we adopt the optimized hyperparameters reported in \citep{du2024tsibench}. For Glocal, we use PatchTST as the backbone. For the matching-based baselines, the Sinkhorn batch size is 256, while those for TDM and PSW-I are 512 and 200, respectively; the corresponding step sizes are 0.001 (Sinkhorn), 0.001 (TDM), and 0.002 (PSW-I). For PSW-I, the unbalanced OT subproblem is solved via majorization--minimization with coefficient 10.0. For SPIRIT, we set the sampling step size to 0.002 and the score-network learning rate to 0.001. We parameterize $s_\theta$ using a three-layer MLP with adaptive layer normalization~\citep{liusundial}, with hidden dimension 256. All experiments are repeated at least three times using four different random seeds to ensure the reliability of the results.

\subsection{Evaluation Metrics}\label{subsec:evaluationMetricResults}
Following previous works~\citep{wang2025optimal,glocalImputation,yang2025towards}, we evaluate our model performance using the MAE and MSE. The detailed definition of these two evaluation metrics are given as follows:
\begin{equation*}
  \text{MAE} \coloneqq \dfrac{\sum_{i=1}^{N}\sum_{j=1}^{T}{\sum_{k=1}^{D}{[| {\mathbf{X}}_{i,j,k}^{\text{ideal}} -  {\mathbf{X}}_{i,j,k}^{\text{imp}} | \odot (\mathbf{1}_{N\times T\times D}-\boldsymbol{M})_{i,j,k}]}}}{\sum_{i=1}^{N}\sum_{j=1}^{T}\sum_{k=1}^{D}{{(\mathbf{1}_{N\times T\times D}-\boldsymbol{M})_{i,j,k}}}},
\end{equation*}

\begin{equation*}
  \text{MSE} \coloneqq \dfrac{\sum_{i=1}^{N}\sum_{j=1}^{T}{\sum_{k=1}^{D}{[\Vert  {\mathbf{X}}_{i,j,k}^{\text{ideal}} -  {\mathbf{X}}_{i,j,k}^{\text{imp}}\Vert_2^2 \odot (\mathbf{1}_{N\times T\times D}-\boldsymbol{M})_{i,j,k}]}}}{\sum_{i=1}^{N}\sum_{j=1}^{T}\sum_{k=1}^{D}{{(\mathbf{1}_{N\times T\times D}-\boldsymbol{M})_{i,j,k}}}},
\end{equation*}

\subsection{Simulation of MCAR Scenario}\label{subsec:simulationMissingDataScenario}
Following established protocols by~\citet{wang2025optimal} and our setting, we simulate the missing data as outlined in reference~\cite{jarrett2022hyperimpute}: Initially, a random subset of features is selected to remain non-missing. The masking of the remaining features is conducted using a logistic model, which employs the non-missing features as predictors. This model is parameterized with randomly selected weights, and the bias is adjusted to achieve the desired missingness rate.

\section{Additional Experimental Results}
\subsection{Empirical Convergence Analysis}

In this subsection, we empirically validate the convergence claims discussed in \Cref{subsec:discussionsOnTheConvergence}. Specifically, on the ETT-h1 dataset with missing ratios $p_{\text{miss}}\in\{0.1,0.2,0.3\}$, we track the score-matching loss $\mathcal{L}^{\text{DSM}}$ together with the imputation metrics (MAE/MSE) during ``Score Learning'' and ``Recursive Imputation'' stages. As shown in \Cref{subfig:conv_score_results_1,subfig:conv_score_results_2,subfig:conv_score_results_3}, $\mathcal{L}^{\text{DSM}}$ decreases steadily and then plateaus, indicating that the score-network optimization reaches a stable basin. Consistently, MAE and MSE in \Cref{subfig:conv_mae_results_1,subfig:conv_mae_results_2,subfig:conv_mae_results_3,subfig:conv_mse_results_1,subfig:conv_mse_results_2,subfig:conv_mse_results_3} exhibit the same trend, stabilizing as training proceeds. Overall, these results provide empirical evidence of SPIRIT’s convergence behavior and support the discussion in \Cref{subsec:discussionsOnTheConvergence}.

\begin{figure}[htbp]
    \centering
      \subfigure[$\mathcal{L}^{\text{DSM}}$ at $p_{\text{miss}}=0.1$ .\label{subfig:conv_score_results_1}]{\includegraphics[width=0.31\linewidth]{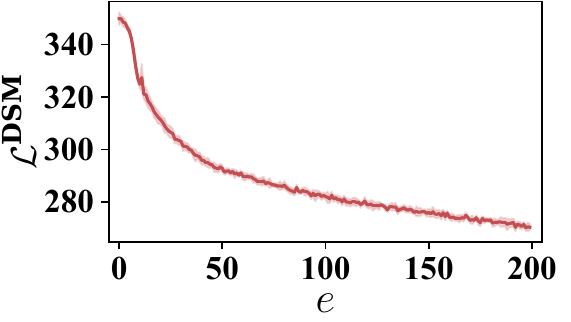}}
         \subfigure[$\mathcal{L}^{\text{DSM}}$ at $p_{\text{miss}}=0.2$ .\label{subfig:conv_score_results_2}]{\includegraphics[width=0.31\linewidth]{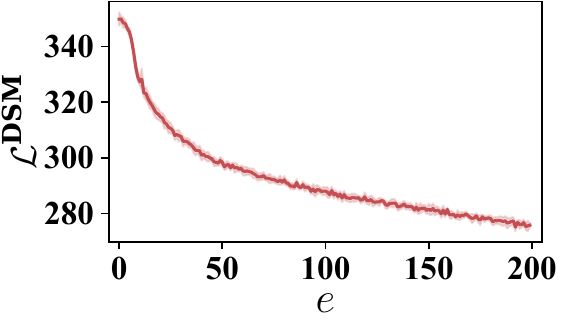}}
       \subfigure[$\mathcal{L}^{\text{DSM}}$ at $p_{\text{miss}}=0.3$ .\label{subfig:conv_score_results_3}]{\includegraphics[width=0.31\linewidth]{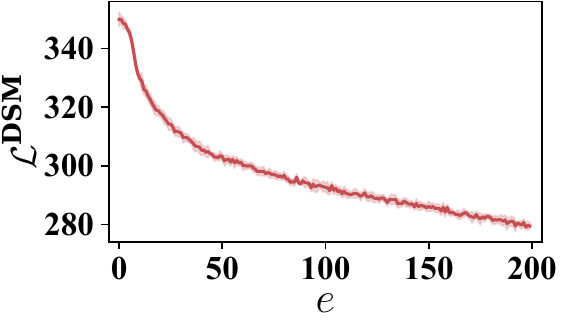}}
      \subfigure[MAE at $p_{\text{miss}}=0.1$ .\label{subfig:conv_mae_results_1}]{\includegraphics[width=0.31\linewidth]{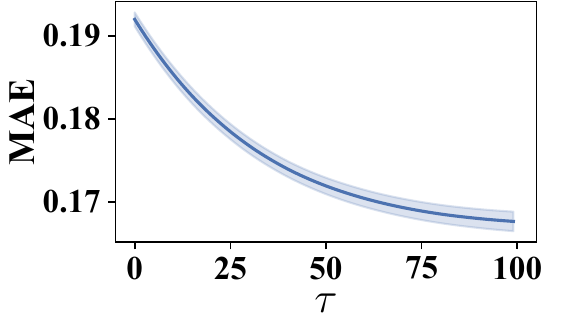}}
         \subfigure[MAE at $p_{\text{miss}}=0.2$ .\label{subfig:conv_mae_results_2}]{\includegraphics[width=0.31\linewidth]{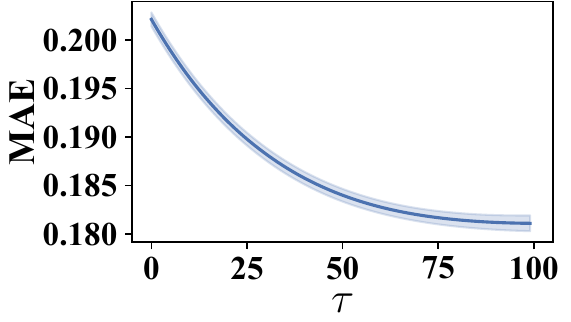}}
       \subfigure[MAE at $p_{\text{miss}}=0.3$ .\label{subfig:conv_mae_results_3}]{\includegraphics[width=0.31\linewidth]{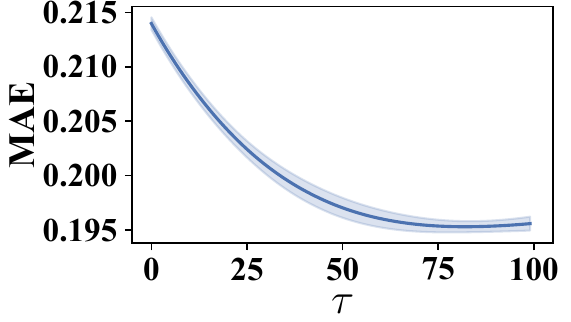}}
      \subfigure[MSE at $p_{\text{miss}}=0.1$ .\label{subfig:conv_mse_results_1}]{\includegraphics[width=0.31\linewidth]{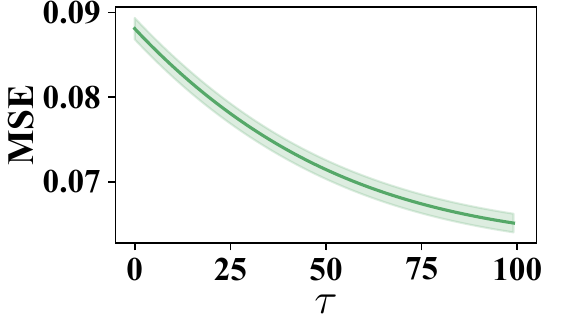}}
         \subfigure[MSE at $p_{\text{miss}}=0.2$ .\label{subfig:conv_mse_results_2}]{\includegraphics[width=0.31\linewidth]{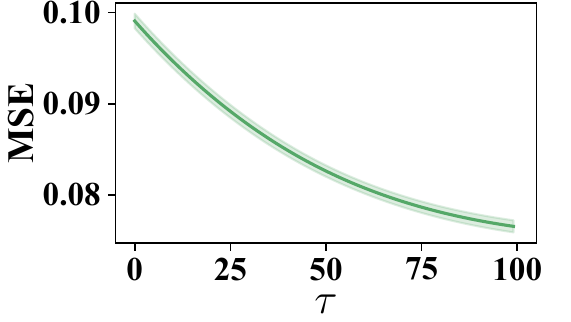}}
       \subfigure[MSE at $p_{\text{miss}}=0.3$ .\label{subfig:conv_mse_results_3}]{\includegraphics[width=0.31\linewidth]{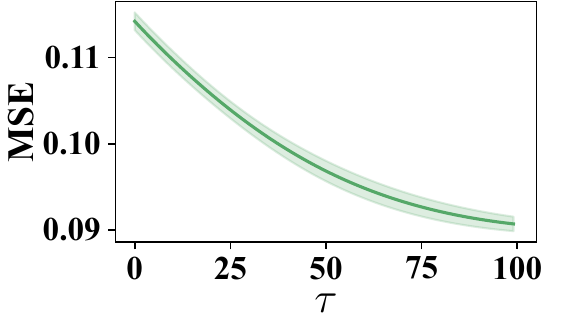}}
    \caption{Transport plan comparison between OT and SPT. The lines and shaded areas indicate the mean and one standard deviation from the mean, respectively.}\label{fig:convFigure}
\end{figure}

\newpage
\subsection{Empirical Time Complexity}
We further report the empirical runtime of SPIRIT in \Cref{fig:timeComplexillustration}. As the dataset size increases, the ``Score Learning'' stage scales more steeply than the ``Recursive Imputation'' stage. This is expected since DSM training requires backpropagating through the network to compute gradients with respect to the inputs. Nevertheless, the overall runtime remains below 3 minutes even for the largest setting, demonstrating the practical time efficiency of SPIRIT.

\begin{figure}[htbp]
  \centering
  \includegraphics[width=0.45\textwidth]{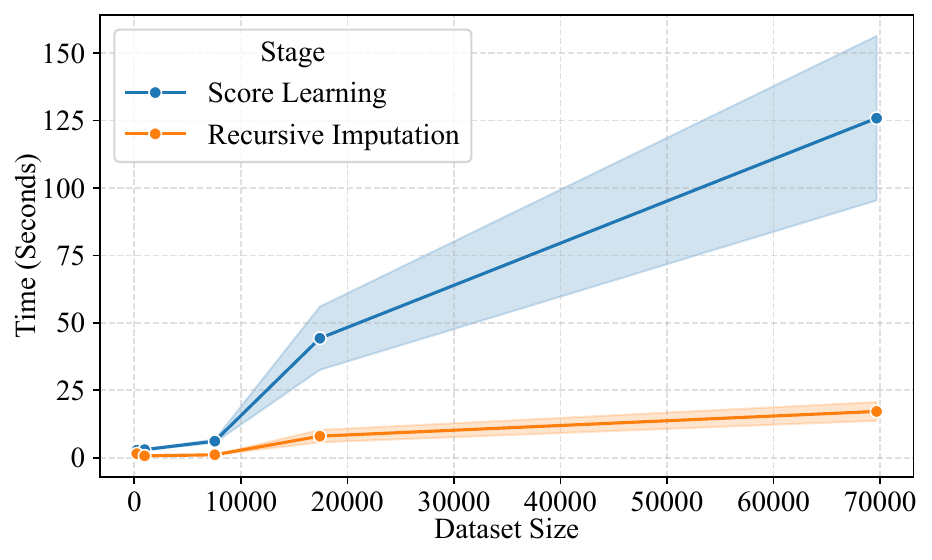} 
  \caption{The computational time for ``Score Learning'' stage and ``Recursive Imputation'' stage. The scatters and shaded areas indicate the mean and one standard deviation from the mean, respectively.}\label{fig:timeComplexillustration}
\end{figure}

\section{Discussions on Limitations \& Future Research Directions}
In our study, although SPIRIT achieves promising performance on TSDI, several limitations remain and motivate future research:
\begin{itemize}[leftmargin=*]
    \item \textbf{Periodic structure and temporal dependence:} Following common practice, we treat multivariate time~\citep{shen2023non,shen2024multi}. While effective, this design does not explicitly model temporal autocorrelation and periodic patterns. Future work could incorporate frequency-domain diffusion models~\citep{crabbetime}, or design a proximal term using PSW~\citep{wang2025optimal} that better preserves periodic structure.
    
    \item \textbf{Limited uncertainty quantification:} To prioritize accurate imputation, our derivation reduces the dissipative component, which may diminish sample diversity and weaken uncertainty estimates. A promising direction is to adopt multi-objective optimization~\citep{gong2022fill} to balance accuracy and diversity, enabling reliable uncertainty quantification without sacrificing imputation quality.
    
    \item{\textbf{Downstream-task awareness:} We focus on TSDI \textit{per se}. Future work should investigate how to integrate SPIRIT with downstream models, for example, forecasting~\citep{qiu2024tfb}, to improve robustness under missing-data scenarios, potentially via joint training or end-to-end task-aware objectives.}
    
    \item{\textbf{Beyond MCAR to MNAR settings:} Our current formulation primarily targets missing completely at random (MCAR). In many real applications, missingness depends on the (unobserved) values or the data-collection mechanism (MNAR). Extending SPIRIT to MNAR settings may require explicitly modeling the missingness mechanism~\citep{kyono2021miracle} and incorporating distributionally robust optimization~\citep{levy2020large} to improve reliability under mechanism shift.}
     \item{\textbf{Alternative score-learning strategies:} In this work, we adopt denoise score matching to learn the score function~\citep{vincent2011connection}. DSM requires computing gradients with respect to the inputs, which incurs additional backpropagation overhead and leads to higher runtime as the data scale increases. Future work could explore alternative score-learning objectives or architectures that reduce input-gradient computation.
     }
\end{itemize}

\end{document}